\documentclass{article}

\usepackage{arxiv}

\usepackage[utf8]{inputenc} 
\usepackage[T1]{fontenc}    
\usepackage{hyperref}       
\usepackage{url}            
\usepackage{booktabs}       
\usepackage{amsfonts}       
\usepackage{nicefrac}       
\usepackage{microtype}      
\usepackage{xcolor}         

\usepackage{mathtools}

\DeclareMathAlphabet{\mathsfit}{\encodingdefault}{\sfdefault}{m}{sl}
\SetMathAlphabet{\mathsfit}{bold}{\encodingdefault}{\sfdefault}{bx}{n}

\def\E{{\bf E}}

\def\0{{\bf 0}}
\def\1{{\bf 1}}

\def\AM{{\mathcal A}}

\def\DM{{\mathcal D}}

\def\FM{{\mathcal F}}
\def\IM{{\mathcal I}}

\def\OM{{\mathcal O}}
\def\PM{{\mathcal P}}
\def\SM{{\mathcal S}}
\def\TM{{\mathcal T}}

\def\RB{{\mathbb R}}
\def\NB{{\mathbb N}}
\def\EB{{\mathbb E}}

\def\argmax{\mathop{\rm argmax}}

\def\ie{{\em i.e.\/}}

\newcommand{\Var}{\mathrm{Var}}

\newcommand{\Tfed}{\mathcal{T}_\text{fed}}

\newcommand{\Tfedn}[1]{\mathcal{T}_{\text{fed-}#1}}

\newcommand{\ind}{\mathds{1}}

\def\AM{{\mathcal A}}

\def\DM{{\mathcal D}}

\def\FM{{\mathcal F}}
\def\IM{{\mathcal I}}

\def\MM{{\mathcal M}}

\def\OM{{\mathcal O}}
\def\PM{{\mathcal P}}
\def\SM{{\mathcal S}}
\def\TM{{\mathcal T}}

\def\RM{{\mathcal R}}

\def\RB{{\mathbb R}}
\def\EB{{\mathbb E}}

\usepackage{amsmath}
\usepackage{amssymb}
\usepackage{mathtools}
\usepackage{amsthm}
\usepackage{ dsfont }
\usepackage{diagbox}

\theoremstyle{plain}
\newtheorem{thm}{Theorem}[section]
\newtheorem{lem}{Lemma}[section]

\newtheorem{asmp}{Assumption}[section]

\usepackage{algorithm}
\usepackage{algorithmic}
\usepackage{bm}

\title{Federated Control in Markov Decision Processes}

\author{%
	\textbf{Hao Jin}\\
	jin.hao@pku.edu.cn\\
	Peking University\\
    \\
	\textbf{Liangyu Zhang}\\
	zhangliangyu@pku.edu.cn\\
	Peking University\\
    \And
    \textbf{Yang Peng}\\
    pengyang@pku.edu.cn\\
    Peking University\\
    \\
	\textbf{Zhihua Zhang}\\
	zhzhang@math.pku.edu.cn\\
	Peking University\\
}

\begin{document}

\maketitle

\begin{abstract}
We study problems of federated control in Markov Decision Processes. To solve an MDP with large state space, multiple learning agents are introduced to collaboratively learn its optimal policy without communication of locally collected experience. In our settings, these agents have limited capabilities, which means they are restricted within different regions of the overall state space during the training process. In face of the difference among restricted regions, we firstly introduce concepts of leakage probabilities to understand how such heterogeneity affects the learning process, and then propose a novel communication protocol that we call Federated-Q protocol (FedQ), which periodically aggregates agents' knowledge of their restricted regions and accordingly modifies their learning problems for further training. In terms of theoretical analysis, we justify the correctness of FedQ as a communication protocol, then give a general result on sample complexity of derived algorithms FedQ-X with the RL oracle 
, and finally conduct a thorough study on the sample complexity of FedQ-SynQ. Specifically, FedQ-X has been shown to enjoy linear speedup in terms of sample complexity when workload is uniformly distributed among agents. Moreover, we carry out experiments in various environments to justify the efficiency of our methods.
\end{abstract}

\section{Introduction}
Recent years have witnessed a rapidly increasing deployment of reinforcement learning (RL) \cite{sutton1998introduction} in the policy-making process of many realistic applications, such as online display advertising \cite{zou2019reinforcement,bai2019model} and routing of automobiles \cite{kiran2021deep,fayjie2018driverless,chen2019model}.
In these scenarios, traditional single-agent RL algorithms are seriously challenged by the scales of learning problems.
Specifically, RL problems are modeled as Markov Decision Processes (MDPs) with a state space $\SM$  and an action space  $\AM$. 
Sizes of the state space $|\SM|$ and the action space $|\AM|$ explosively grow in the real-world applications mentioned above, while the sample complexity of a single-agent RL algorithm linearly increases with respect to $|\SM|$ and $|\AM|$ \cite{li2023q,chen2020finite}.
To speed up the learning process, it is appealing to consider multi-agent variants of RL algorithms \cite{wang2019distributed,martinez2019decentralized}.
In some other tasks, the policy-making process is naturally distributed among multiple separate agents \cite{gedel2021low,habibi2019comprehensive}.
For instance, the task of data transmission has to be sequentially processed by multiple cell towers with limited coverage distance \cite{ahamed20215g,dhekne2017extending}.
What makes things worse is that communication among learning agents has to consider issues of privacy protection, which makes any gathering of locally collected data infeasible.
This motivates us to design communication protocols for 
distributed RL problems in a setting that we call the federated control.

In particular, we  formulate the federated control problem for distributed RL as a six-tuple $\langle \SM,\AM,\PM,\RM,\gamma,\{\SM_k\}_{k=1}^N\rangle$, where $N$ indicates the number of distributed agents and $\SM_k\subseteq\SM$ represents the subset of space the $k$-th agent has access to.
Specifically, having access to $\SM_k$ indicates that the $k$-th agent is able to collect experience and execute its policy when the state falls within $\SM_k$, and the data collected by the $k$-th agent is therefore notated as $\DM_{\SM_k}$.
These agents are expected to collaboratively learn the optimal policy $\pi^\star$ in maximizing cumulative reward function of the MDP $\langle \SM,\AM,\PM,\RM,\gamma\rangle$.
Prevalent multi-agent variants \cite{liu2019lifelong,nadiger2019federated,wang2020federated,jin2022federated} assume that agents have access to the overall state set $\SM$, which leads to a special case of our federated control problems with $\SM_k=\SM,~\forall k\in[N]=\{1,\dots,N\}$.
In their settings, these agents periodically aggregate either their policy gradients or their local Q functions to accelerate the training process.
Such assumption on agents simplifies the design of communication protocols and naturally guarantees the convergence of algorithms \cite{jin2022federated,woo2023blessing}.
Our formulation of federated control problems emphasizes the heterogeneity among $\{\SM_k\}_{k=1}^N$, where agents are specialized in collecting data and executing policies within certain set of states.

Heterogeneity among subsets of states challenges traditional designs of communication protocols, since agents face related but different learning problems.
Inspired by techniques in decomposing large-scale MDPs \cite{meuleau1998solving,fu2015optimal,daoui2010exact}, we believe that these agents are able to learn the globally optimal policy when provided with properly designed knowledge from others.
In our settings, agents restricted in different regions lack information about states beyond their collected experience, while they are specialized in extracting information of local policies within their subsets of states.
Therefore, we schedule local agents to periodically upload their Q functions based on their locally learned optimal policies, and then utilize these information to guide later training of these agents.
Such methodology is referred to as Federated-Q protocol, \textbf{FedQ}, in the following discussion.
It is worth noting that FedQ does not restrict how local agents update their policies, but requires them to keep timely Q functions of their local policies, which is prevalent in popular RL algorithms, such as critic functions in those policy-based solutions.
Although FedQ as a communication protocol looks simple at the very first sight, we will justify its correctness and show its generality in deriving theoretically-efficient algorithms in terms of sample complexity.

Existing tools for analysing Q function based RL methods are improper for the theoretical justification of FedQ due to the introduction of heterogeneity among $\{\SM_k\}_{k=1}^N$.
To address the above challenge, we formulate FedQ with a federated Bellman operator $\Tfed$, which is composed of $N$ local Bellman operators in $N$ local MDPs.
Specifically, a local MDP represents the learning problem within the restricted region given the auxiliary information broadcast by FedQ, in which the process of local updates is formulated as a local Bellman operator; the federated Bellman operator tracks the performance of aggregated policies at each communication round.
Suppose the RL oracle utilized by local agents in solving local MDPs is $X$.
Accompanied with the communication protocol FedQ, such algorithm for federated control problems is abbreviated as FedQ-X.
For the algorithm FedQ-X, we provide a general result of its sample complexity.
Specifically, we have shown that FedQ-SynQ, which utilizes synchronous Q-learning as the RL oracle, manages to achieve lower sample complexity at the expense of more rounds of communication.
Moreover, FedQ-X has been shown to assist any participating agent in achieving $N$-times lower sample complexity, when the workload is distributed among different agents, \ie, $|\SM_k|=|\SM|/N,~\forall k\in[N]$.

In summary, our paper offers the following contributions:
\begin{itemize}
\setlength\itemsep{0.1em}
    \item We are the first to consider federated control problems with heterogeneity among state subsets within which different agents are respectively restricted.
    This setting better matches realistic scenarios with distributed deployment and concerns on privacy protection.
    Additionally, we introduce the concept of leakage probabilities to depict how such heterogeneity affects the process of policy learning.
    \item We propose an efficient communication protocol, FedQ, to solve federated control problems.
    FedQ requires agents to communicate knowledge of states within their restricted regions, and utilizes these information for local training.
    \item We theoretically justify the correctness of FedQ, propose a general result for sample complexity of any specific algorithm FedQ-X, and tighten the sample complexity of FedQ-SynQ at the expense of more rounds of communication.
    Moreover, the linear speedup of FedQ-X w.r.t. sample complexity is theoretically justified when workload is uniformly distributed
\end{itemize}

\section{Preliminaries}
In this section, we start from the introduction of classical control problems in MDPs, then briefly review basic knowledge about Q-learning and finally present synchronous Q-learning as a thoroughly studied RL oracle.
\subsection{Classical control in Markov Decision Processes}
An MDP is usually formulated as a five-tuple $\MM=\langle\SM,\AM,\RM,\PM,\gamma\rangle$: $\SM$ and $\AM$ represent the state space and the action space, $\RM$ models the reward function, $\PM$ stands for the transition dynamic, and $\gamma$ serves as the discounted factor.
A classical control problem is to find the optimal policy $\pi^\star_\SM$ maximizing cumulative reward:
\begin{equation*}
\pi^\star_\SM=\arg\max_{\pi_\SM} J(\pi_\SM),
\end{equation*}
\begin{equation}
\label{pre:obj_func}
J(\pi)=\mathbb{E}_{\PM,\pi}\Bigg [ \sum_{t=0}^\infty \gamma^t\RM(s_t,a_t)\Big|s_0\sim d_0\Bigg],
\end{equation}
where expectations are evaluated w.r.t. the randomness of both transition dynamic $\PM$ and the policy $\pi$, $d_0\in\Delta(\SM)$ indicates the distribution of initial states $s_0$, and the footnote $\SM$ of $\pi_\SM$ indicates that the learning agent is allowed to collect experience $D_\SM$ on the overall state space $\SM$ and train its policy $\pi_\SM$ accordingly.
Specifically, an agent adopting Q-value based algorithms is able to collect experience replay buffer $D_\SM=\{(s,a,r,s')|s\in\SM\}$ and an agent with policy gradient methods is able to collect trajectories $D_\SM=\{\tau:=(s_t,a_t,r_t,s_{t+1})_{t=0}^\infty|s_t\in\SM\}$.

\subsection{Basics of Q-learning}
Q-learning represents a class of online reinforcement learning algorithms based on the modification of value functions and Q functions.
For any feasible policy $\pi$, we are able to derive its value function in an MDP $\MM=\langle\SM,\AM,\RM,\PM,\gamma\rangle$ for all $s\in \SM$
\begin{equation}
\notag
V^\pi(s):=\EB_{\PM,\pi}\left [ \sum_{t=0}^\infty \gamma^t \RM(s_t,a_t)\Big|s_0=s\right ],
\end{equation}
and Q function for all $(s,a)\in\SM\times\AM$
\begin{equation}
\notag
    Q^\pi(s,a):=\EB_{\PM,\pi}\Big [\sum_{t=0}^\infty \gamma^t \RM(s_t,a_t)\Big|s_0=s,a_0=a\Big].
\end{equation}
The optimal value function and the optimal Q function are defined respectively as
\begin{equation}
    V^\star(s)=\max_\pi V^\pi(s),~~Q^\star(s,a)=\max_\pi Q^\pi(s,a).
\end{equation}
It is well-known that an optimal deterministic policy $\pi^\star$ can be accordingly derived as
\begin{equation*}
    \pi^\star(s) = \arg\max_{a\in\AM} Q^\star(s,a).
\end{equation*}
To achieve the optimal Q function $Q^\star$, various algorithms are based on the optimistic Bellman operator $\TM$ defined as follows for all $(s,a)\in\SM\times\AM$:
\begin{equation}
\label{pre:optimBellman}
    \TM(Q)(s,a):=\RM(s,a)+\gamma\EB_{s'\sim\PM(\cdot|s,a)}\left[\max_{a'\in\AM}Q(s',a')\right].
\end{equation}
Furthermore, the operator $\TM$ represents a contraction mapping, which means for all $Q_1,Q_2\in\RB^{|\SM||\AM|}$
\begin{equation*}
    \|\TM(Q_1)-\TM(Q_2)\|_\infty\leq \gamma \|Q_1-Q_2\|_\infty,
\end{equation*}
where $\|\cdot\|_\infty$ represents the maximum norm.
The optimal Q function $Q^\star$ is the unique stationary point of this operator $\TM$, \ie, $Q^\star = \TM(Q^\star)$.

\subsection{Synchronous Q-learning}
\label{pre:SyncQ}
Synchronous Q-learning, a thoroughly studied variant of Q-learning, utilizes a generative model to iteratively learn the optimal Q function $Q^\star$.
Specifically, take $Q_t\in\RB^{|\SM||\AM|}$ as the Q function at the $t$-th iteration,
and the synchronous setting allows us to generate $s_t(s,a)\sim\PM(\cdot|s,a)$ for every $(s,a)\in\SM\times\AM$.
We are next to introduce an empirical operator $\TM_t$ as follows:
\begin{align*}
    \TM_t(Q_{t-1})(s,a)&:=\RM(s,a)+\gamma\max_{a'\in\AM}Q_{t-1}(s_t(s,a),a'),
\end{align*}
and accordingly update all state-action values via the following update rule:
\begin{equation}
\label{pre:SynQ_update}
    Q_{t}=(1-\eta_t)Q_{t-1}+\eta_{t-1} \TM_t(Q_{t-1}),
\end{equation}
where $\eta_t$ serves as the learning rate in the $t$-th iteration.
Note that the empirical operator $\TM_t$ is an unbiased estimate of the optimistic Bellman operator $\TM$ defined in Eq. (\ref{pre:optimBellman}), which means $Q_t$ is meant to converge towards the optimal Q function $Q^\star$.

\section{Federated Control in MDPs}
In this section, we firstly formulate problems of federated control in MDPs, and then propose the notion of leakage probability to quantify degrees of heterogeneity in federated control problems.

\subsection{General Problem Formulation}
Problems of federated control are formulated with a six-tuple $\langle \SM,\AM,\RM,\PM,\gamma,\{\SM_k\}_{k=1}^N\rangle$, which is composed of a classical control problem $\MM=\langle \SM,\AM,\RM,\PM,\gamma\rangle$ and a set of restricted regions $\{\SM_k:\SM_k\subseteq \SM\}_{k=1}^N$ for the $N$ involved agents.
Moreover, a feasible federated control problem requires  $\cup_{k=1}^N\SM_k=\SM$, which means that these $N$ restricted regions cover the overall state space.
Similar to classical control, federated control focuses on maximizing the cumulative reward in the MDP $\MM$.
However, instead of learning the policy $\pi_\SM$ on the overall data $D_\SM$, federated control requires these $N$ agents to learn $N$ local policies $\{\pi_{\SM_k}\}_{k=1}^N$ based on local experience restricted within their assigned regions $\{D_{\SM_k}\}_{k=1}^N$.
The federated control problem is formulated as follows:
\begin{align}
    \notag
    \pi^\star_\text{Fed}&=\arg\max_{\pi\in\Pi_\text{Fed}} J(\pi),\\
    \Pi_\text{Fed}:=\{\Phi(\pi_{\SM_1},\dots,&\pi_{\SM_N}):\pi_{\SM_k} \text{ is trained on } D_{\SM_k}\},
\end{align}
where $\Phi$ indicates a predefined strategy merging $N$ local policies to a global policy, and $J(\cdot)$ follows the same definition in the classical control problem described in Eq. \eqref{pre:obj_func}.
\subsection{Leakage Probabilities in Federated Control}
To quantify the hardness of solving federated control problems, we introduce the concept of leakage probabilities for any feasible federated control problem.
The leakage probability reflects the strength of connection among different regions $\{\SM_k\}_{k=1}^N$ under the given transition dynamic $\PM$, which plays an important role in the analysis of FedQ.

For any restricted region $\SM_k$, we evaluate the leakage probability for any state-action pair $(s,a)\in\SM_k\times\AM$:
\begin{equation*}
\PM_k(s,a):=\sum_{s'\notin\SM_k}\PM(s'|s,a).
\end{equation*}
Based on values of leakage probability, we manage to classify state-action pairs $(s,a)\in\SM_k\times\AM$ into two groups, the kernel $K_k$ and the edge  $E_k$:
\begin{align*}
    K_k:=\{(s,a)\in\SM_k\times\AM:\PM_k(s,a)=0\},\\
    E_k:=\{(s,a)\in\SM_k\times\AM:\PM_k(s,a)>0\}.
\end{align*}
All state-action pairs in $E_k$ have non-zero probability of transiting out of $\SM_k$.
To depict the connection between $\SM_k$ and $\SM\backslash\SM_k$, we introduce $p_{\max}^k$ as the maximum leakage probability for the subset $\SM_k$:
\begin{align*}
    p_{\max}^k:=\max_{(s,a)\in E_k}\PM_k(s,a).
\end{align*}
For the $k$-th restricted region, a larger $p_{\max}^k$ indicates that there exists a state-action pair with higher probabilities of transiting to other regions.
For the problem of federated control, we are able to construct sets of maximum leakage probabilities $\{p_{\max}^k\}_{k=1}^N$.

Maximum leakage probabilities quantify the strength of connection among subsets of states $\{\SM_k\}_{k=1}^N$ under certain transition probability $\PM$, which implicitly determines the relevance among learning problems of different agents.
Specifically, a smaller $p_{\max}^k$ indicates a more closed transition dynamic within $\SM_k$ and a more independent learning process for the $k$-th agent.
Consider an extreme case where every involved agent is able to access the whole state space $\SM$, the edging group for any region remains empty $E_k=\emptyset$, and maximum leakage probabilities $p_{\max}^k$ are accordingly set to $0$.
Zero leakage probabilities match the relative independence among these agents, where any single agent is able to learn the optimal policy without any communication.

\section{FedQ: Federated-Q Protocol}
In this section, we firstly introduce concepts of augmented local MDPs for $N$ agents, and then propose the communication protocol of FedQ as a solution to federated control problems.

\subsection{Augmented Local Markov Decision Process}
In the federated control problem, participated agents are not allowed to access experience collected outside their restricted regions, which makes the overall MDP $\MM$ improper for the algorithm design and analysis of local policy learning.
Given a federated control problem as $\langle \SM,\AM,\RM,\PM,\gamma,\{\SM_k\}_{k=1}^N\rangle$, we construct augmented local MDPs $\{\MM_k\}_{k=1}^N$ for participated agents.
Specifically, we introduce a predefined vector $\tilde{V}\in\RB^{|\SM|}$, augment the state space with an absorbing state $\tilde{\SM}:=\SM\cup\{s_{null}\}$, and formulate the $k$-th local MDP $\MM_k$ as $\langle \tilde{\SM},\AM,\RM_k,\PM_k,\gamma,\tilde{V}\rangle$ where the augmented reward $\RM_k$ is defined as 
\begin{equation*}
    \RM_k(s,a)=
    \begin{cases}
    \RM(s,a)&s\in\SM_k, \\
    \tilde{V}(s)&s\in\SM\backslash\SM_k, \\
    0&s=s_{null},
    \end{cases}
\end{equation*}
and the augmented dynamic $\PM_k$ is defined as
\begin{equation*}
    \PM_k(s'|s,a)=
    \begin{cases}
    \PM(s'|s,a)&s\in\SM_k, \\
    1&s'=s_{null}~\text{and}~s\notin \SM_k, \\
    0&\text{otherwise}.
    \end{cases}
\end{equation*}
Put simple, $\MM_k$ keeps the reward function and transition dynamics within the restricted region $\SM_k$ unchanged, and treats states outside $\SM_k$ as absorbing states with terminal rewards according to the given vector $\tilde{V}$.
The construction of $\{\MM_k\}_{k=1}^N$ depends on the value of $\tilde{V}$ and a proper $\tilde{V}$ leads to rational augmented local MDPs.

\begin{algorithm}[!htb]
    \caption{FedQ with an RL oracle $X$ (FedQ-X)}
    \begin{algorithmic}
    \label{Alg:FedQ_full}
        \STATE \textbf{Initialization:} $\langle \SM,\AM,\RM,\PM,\gamma,\{\SM_k\}_{k=1}^N\rangle$ as the federated control problem,$~Q_0$ as the initial global Q functions.
        \FOR{ $r=0$ {\bfseries to} $R-1$ }
            \STATE $V_r(s)\leftarrow\max_{a\in\AM} Q_r(s,a),\forall s\in\SM$.
            \STATE Construct $\{\MM_{r,k}\}_{k=1}^N$ based on $\tilde{V}=V_r$.
                \FOR{ $k=1$ {\bfseries to} $N$ }
                    \STATE Update the local policy $\pi_{r,k}$ to maximize cumulative reward in $\MM_{r,k}$ via calling $X$ for $E$ iterations.
                    \STATE Obtain Q functions of the local policy as $Q_{r,k}$.
                \ENDFOR
            \STATE Aggregate $\{Q_{r,k}\}_{k=1}^N$ as $Q_{r+1}$ following the rule described in Eq. \eqref{Alg:FedQ}.
        \ENDFOR
        \STATE \textbf{return} $Q_R$.
    \end{algorithmic}
\end{algorithm}

\subsection{Federated-Q Protocol}
We propose Federated-Q protocol (FedQ) as a communication protocol to solve federated control problems, which involves Q functions of policies locally learned in $\{\MM_k\}_{k=1}^N$.
To exchange knowledge between different restricted regions, FedQ reloads local MDPs $\{\MM_k\}_{k=1}^N$ via periodically updating values of $\tilde{V}$.
Suppose the values of $\tilde{V}$ after the $r$-th round of communication are denoted as $V_r$, and the local MDP constructed for the $k$-th agent based on $\tilde{V}=V_r$ is denoted as $\MM_{r,k}$.
In terms of local updates, agents are expected to optimize their local policies in local MDPs respectively with any feasible oracle.
After $E$ steps of local updates, FedQ requires agents to communicate Q functions $\{Q_{r-1,k}\}_{k=1}^N$ of their locally learned policies for the $r$-th round of communication.
During the $r$-th communication round, $\{Q_{r-1,k}\}_{k=1}^N$ are aggregated as follows for the derivation of $V_{r}$:
\begin{align}
    \label{Alg:FedQ}
    Q_{r}(s,a)=\frac{1}{N(s)}&\sum_{k=1}^NQ_{r-1,k}(s,a)\mathds{1}_k(s),\forall a\in\AM,\\
    \notag
    V_{r}(s) &= \max_a Q_r(s,a),\forall s\in\SM,
\end{align}
where $N(s):=\sum_{k=1}^N\mathds{1}_k(s)$ denotes the number of regions containing the state $s$.
Afterwards, values of $V_r$ are broadcast among $N$ agents during the $r$-th communication round to accordingly construct $\{\MM_{r,k}\}_{k=1}^N$ for the next $E$ iterations of local updates.
Full implementation of FedQ is shown in Algorithm \ref{Alg:FedQ_full}.

FedQ is regarded as a communication protocol, because it does not specify the type of oracles utilized by local agents for local updates.
Its only requirement on local agents is to timely update their Q functions in local MDPs, which naturally holds for popular RL algorithms.
Specifically, Q-value based algorithms keep Q functions of the current policy for following updates, while most of policy-gradient based algorithms maintain value (Q) functions for the reduction on variation of policy updates.
When FedQ is adopted as the communication protocol and local agents utilize RL algorithm $X$ as the oracle mentioned in Algorithm \ref{Alg:FedQ_full}, the federated algorithm is shortened as FedQ-X for the following theoretical analysis.

\section{Analysis of FedQ}
In this section, we theoretically analyse mechanisms of the proposed communication protocol, FedQ, in converging to the globally optimal Q function, and then gives the sample complexity for a detailed algorithm, FedQ-SynQ.
Detailed proofs of lemmas and theorems in this section are left in Appendices \ref{Appendix:FedOp}, \ref{Appendix:SampleComplexityFedQX}, and \ref{section:proof_fedq_synq}.

\subsection{Convergence of FedQ}
To analyse the convergence of FedQ, we conduct two-step theoretical analysis without specifying detailed RL oracles.
In order to investigate effects of the communication protocol on convergence, we ignore approximation error in deriving optimal policies of local MDPs.
Processes of local updates and Q-value aggregation are respectively modeled as local Bellman operators and federated Bellman operators.
\subsubsection{Local Bellman Operators}
The concept of local Bellman operators $\{\Tfed^k\}_{k=1}^N$ is introduced to model the progress of local updates for these $N$ agents.
For the $k$-th agent, $\Tfed^k$ formulates the mapping from the aggregated Q function $Q_r$ to the optimal Q functions $Q_{r,k}$ for its local MDP $\MM_{r,k}$:
\begin{equation*}
    \Tfed^k(Q_r) = Q_{r,k}.
\end{equation*}
Since $Q_{r,k}$ represents Q functions of the optimal policy in $\MM_{r,k}$, it has to satisfy the following Bellman optimality equation for all $(s,a)\in\SM_k\times\AM$:
\begin{align}
\notag
    Q_{r,k}(s,a)&=\RM_k(s,a)+\gamma\EB_{\PM_k}\max_{a'}Q_{r,k}(s',a')\\
\notag
    &=\RM(s,a) + \gamma\EB_{\PM}\Big [\mathds{1}_k(s')\max_{a'}Q_{r,k}(s',a')\\
    \label{ana:Tfed_local_optimal}
    &~~~~~+(1-\mathds{1}_k(s'))V_r(s')\Big ].
\end{align}
We are next to show the property of contraction for the operator $\Tfed^k$ in federated control problems.
\begin{lem}
[Contraction property of $\Tfed^k$]
\label{ana:Tfedk_contraction}
For any local Bellman operator $\Tfed^k$, it satisfies the contraction property as follows:
    \begin{align}
        \max_{s\in\SM_k,a\in\AM}\Big|\Tfed^k(Q_1)(s,a)-\Tfed^k(Q_2)(s,a)\Big|\leq \\
        \gamma_\text{fed}^k\|Q_1-Q_2\|_\infty,
        \notag~\forall Q_1,Q_2\in\RB^{|\SM||\AM|},
    \end{align}
    where $\gamma_\text{fed}^k=\frac{\gamma p^k_{\max}}{1-\gamma (1-p^k_{\max})}\leq\gamma$.
\end{lem}
\textit{\textbf{Remark}} 1: 
Although $\Tfed^k$ satisfies the contraction property, it is not a traditionally known contraction mapping as optimistic Bellman operators $\mathcal{T}$.
In fact, $\Tfed^k$ is an idempotent operator, \ie,  $\left(\Tfed^k\right)^2(Q_r)=\Tfed^k(Q_r)$.
\subsubsection{Federated Bellman Operator}
To study the convergence of FedQ, we focus on the global Q function obtained via aggregating local Q functions during each communication round.
We introduce federated Bellman operator $\Tfed$ to formulate the relationship between global Q functions of adjacent communication rounds:
\begin{equation*}
    \Tfed(Q_r)=Q_{r+1}.
\end{equation*}
Updates from $Q_r$ to $Q_{r+1}$ are indirect, since they do not depend on direct sampling from the global MDP.
According to the formulation of FedQ, $\Tfed$ is in fact made up of local Bellman operators as follows: 
\begin{equation*}
    \Tfed(Q)(s,a)=\frac{1}{N(s)}\sum_{k=1}^N\Tfed^k(Q)(s,a)\mathds{1}_k(s).
\end{equation*}
The property of $\Tfed$ is strongly dependent on properties of $\{\Tfed^k\}_{k=1}^N$.
In fact, as long as $\{\Tfed^k\}_{k=1}^N$ satisfy the contraction property, $\Tfed$ is also a contraction mapping.
We formalize the relationship with the following lemma.
\begin{lem}[Contraction property of $\Tfed$]\label{lem:Tfed_contraction}
$\Tfed$ is a contraction mapping, \ie
\begin{align}
    \|\Tfed(Q_1)-\Tfed(Q_2)\|_\infty\leq\gamma_\text{fed}\|Q_1-Q_2\|_\infty,\\
    \notag
    \forall Q_1,Q_2\in\RB^{|\SM||\AM|},
\end{align}
\vspace{-0.1in}
where $\gamma_\text{fed}=\max_k\gamma_\text{fed}^k\leq\gamma$.
\end{lem}
\textit{\textbf{Remark}} 2: The dependence of $\gamma_{\text{fed}}$ on $\{\gamma_{\text{fed}}^k\}_{k=1}^N$ implies that uniformly smaller $\{p_{\max}^k\}_{k=1}^N$ result in a smaller $\gamma_{\text{fed}}$.
This matches our intuition that FedQ achieves faster convergence when federated control problems exhibit more closed transition dynamics within restricted regions $\{\SM_k\}_{k=1}^N$.

We are next to show that $\Tfed$ has a unique stationary point corresponding to the globally optimal Q functions.
\begin{lem}
    [Stationary point of $\Tfed$]\label{lem:Tfed_fixpoint} 
    Given a federated control problem $\langle \SM,\AM,R,\PM,\gamma,\{\SM_k\}_{k=1}^N\rangle$,
    $\MM=\langle \SM,\AM,R,\PM,\gamma\rangle$ is its related MDP, and $Q^\star$ is its optimal Q function satisfying Eq. \ref{pre:optimBellman}.
    Then $Q^\star$ is the unique stationary point of $\Tfed$.
\end{lem}
Therefore, we justify the correctness of our proposed communication protocol, FedQ.
As long as the RL oracle $X$ is accurate enough, global Q functions aggregated during communication rounds converge as if a super agent with access to the global state space applies optimistic Bellman operator in updating its Q functions.

\subsection{Sample Complexity of FedQ-X}
Combining the communication protocol with any feasible RL oracle $X$, we get algorithms for federated control problems as FedQ-X.
Even when the RL oracle $X$ is unable to derives the exact solution shown in \ref{ana:Tfed_local_optimal}, FedQ-X is also guaranteed to converge to the optimal Q functions.
\begin{lem}
    [Convergence of FedQ-X]\label{lem:converge_fedqx}
    Suppose $Q^\star$ is the Q function of the optimal policy in the MDP $\MM$, and
    the RL oracle $X$ provides sub-optimal local policies for local MDPs with entrywise $\varepsilon_r$-accurate Q functions at the $r$-th communication round.
    Then the global Q function $Q_R$ in FedQ-X satisfies
    \begin{equation*}
        \begin{aligned}
            \limsup_{R\rightarrow\infty}\|Q^\star-Q_R\|_\infty&\leq\limsup_{R\rightarrow\infty}\sum_{r=1}^R\gamma_{\text{fed}}^{R-r}\varepsilon_r\\ &\leq\frac{1}{1-\gamma_{\text{fed}}}\limsup_{R\rightarrow\infty}\varepsilon_R.
        \end{aligned}
    \end{equation*}
    Furthermore, if $\varepsilon_R\to0$, we have $Q_R$ converges to $Q^\star$.
\end{lem}
Additionally, we manage to provide the following upper bound for sample complexity of FedQ-X.
\begin{asmp} 
    [Sample Complexity of the oracle X]
    \label{ana:X_sample}
    Suppose the target MDP is formulated as $\langle \SM,\AM,R,\PM,\gamma\rangle$, and $Q^\star$ is the Q function of the optimal policy.
    Consider for any $\delta\in(0,1)$ and $\varepsilon\in(0,\frac{1}{1-\gamma})$.
    In order to produce a policy whose Q function $Q$ satisfies $\|Q^\star-Q\|_\infty\leq\varepsilon$ with probability at least $1-\delta$, the RL oracle $X$ requires following sample complexity
    \begin{equation*}
        \tilde{O}(\frac{|\SM||\AM|}{\varepsilon^2(1-\gamma)^m})
    \end{equation*}
    where $m\geq 3$ and $\tilde{O}$ ignores logarithmic terms.
\end{asmp}
\textit{\textbf{Remark}} 3: Sample complexity of $X$ is lower bounded by $\tilde{O}(\frac{|\SM||\AM|}{\varepsilon^2(1-\gamma)^3})$, \ie, $m=3$, which is achieved by variance-reduced Q-learning \cite{wainwright2019variance} and other model-based methods \cite{agarwal2020model, li2023breaking}.

\begin{thm}
    [Sample Complexity of FedQ-X]
    \label{Thm:SampleComplexity_FedQX}
    Suppose the federated control problem is formulated as $\langle \SM,\AM,\RM,\PM,\gamma,\{\SM_k\}_{k=1}^N\rangle$, $Q^\star$ is the optimal Q function of the corresponding MDP, and $p_{\max}^k>0$ for all $k\in[N]$.
    Consider the RL oracle $X$ satisfies Assumption \ref{ana:X_sample}, and the initialization of global Q function obeys $0\leq Q_0(s,a)\leq\frac{1}{1-\gamma},\forall (s,a)\in\SM\times\AM$.
    FedQ-X requires $R=\tilde{O}(\frac{p_{\max}}{1-\gamma})$ rounds of communication and the $k$-th agent requires $$\tilde{O}\big(\frac{p_{\max}^3|\SM_k||\AM|}{\varepsilon^2(1-\gamma)^{m+3}}\big)$$ samples to achieve $\|Q^\star-Q_R\|\leq\varepsilon$ with probability at least $1-\delta$.
\end{thm}

\textit{\textbf{Remark}} 4: 
The number of communication rounds for FedQ-X depends on the value of $p_{\max}$.
This matches our previous analysis on contraction mapping $\Tfed$.

Suppose the RL oracle is synchronous Q-learning whose sample complexity is given by $\tilde{O}(\frac{|\SM||\AM|}{\varepsilon^2(1-\gamma)^4})$.
Following Theorem \ref{Thm:SampleComplexity_FedQX}, FedQ-SynQ is expected to have a sample complexity of $\tilde{O}(\frac{p_{\max}^3|\SM||\AM|}{\varepsilon^2(1-\gamma)^7})$.
We are next to show that sample complexity of FedQ-SynQ can be tightened at the expense of more rounds of communication, which matches theoretical results of \cite{woo2023blessing} and does not contradict with Theorem \ref{Thm:SampleComplexity_FedQX}.
\begin{thm}
    [Sample Complexity of FedQ-SynQ]\label{Thm:samplecomplexity_fedsynq}
    Suppose a federated control problem is formulated as $\langle \SM,\AM,\RM,\PM,\gamma,\{\SM_k\}_{k=1}^N\rangle$ ,
    $Q^*$ is the optimal Q function of the corresponding MDP, and $p_{\max}^k>0$ for all $k\in[N]$.
    Consider the initialization of global Q function obeys $0\leq Q_0(s,a)\leq\frac{1}{1-\gamma},\forall (s,a)\in\SM\times\AM$.
    FedQ-SynQ requires $\tilde{O}(\frac{1}{1-\gamma}\max\{\frac{1}{1-\gamma},N_{\min}\})$ rounds of communication and the $k$-th agent requires
    $$
    \tilde{O}\big(\frac{|\SM_k||\AM|}{N_{\min}\varepsilon^2(1-\gamma)^5}\big)
    $$
    samples to achieve $\|Q^\star-Q_R\|\leq\varepsilon$ with probability at least $1-\delta$, where $N_{\min}:=\min_{s\in\SM}N(s)$.
\end{thm}

It is worth noting that the dependence of sample complexity for the $k$-th agent on the size of its restricted region $|\SM_k|$.
FedQ-X, including FedQ-SynQ, achieves linear speedup w.r.t. sample complexity for any agent, when the workload is uniformly distributed among agents, \ie,  $|\SM_k|=O(|\SM|/N)$.
Specifically, each of agents utilizing FedQ-X enjoys a sample complexity of $\tilde{O}\big(\frac{p_{\max}^3|\SM||\AM|}{N\varepsilon^2(1-\gamma)^{m+3}}\big)$.

\begin{figure*}[t]
    \centering
    \includegraphics[width=0.6\textwidth]{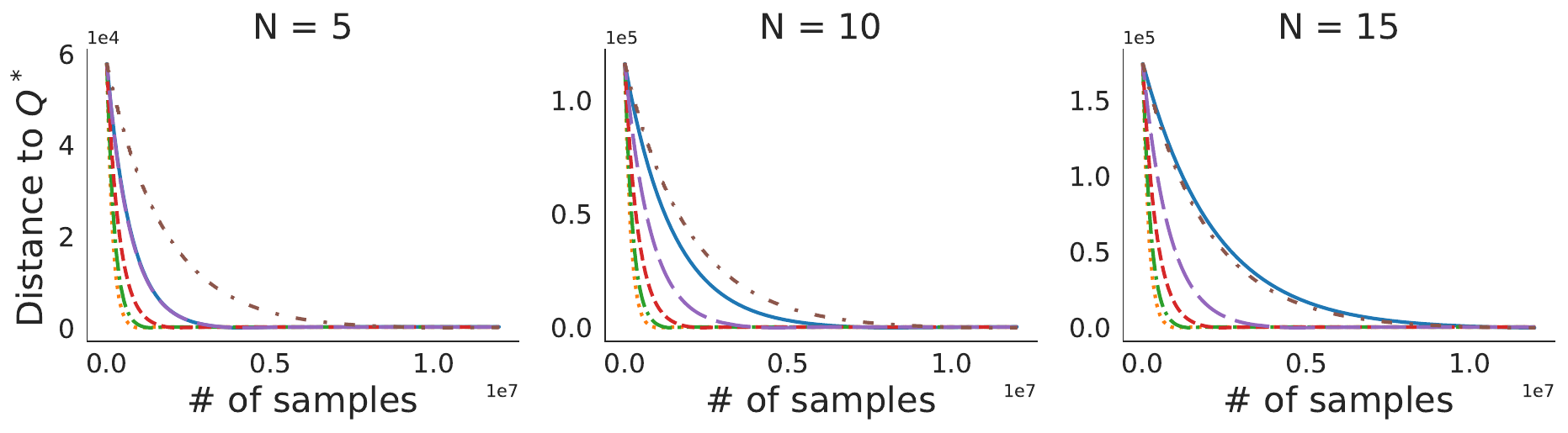}
    \includegraphics[width=0.8\textwidth]{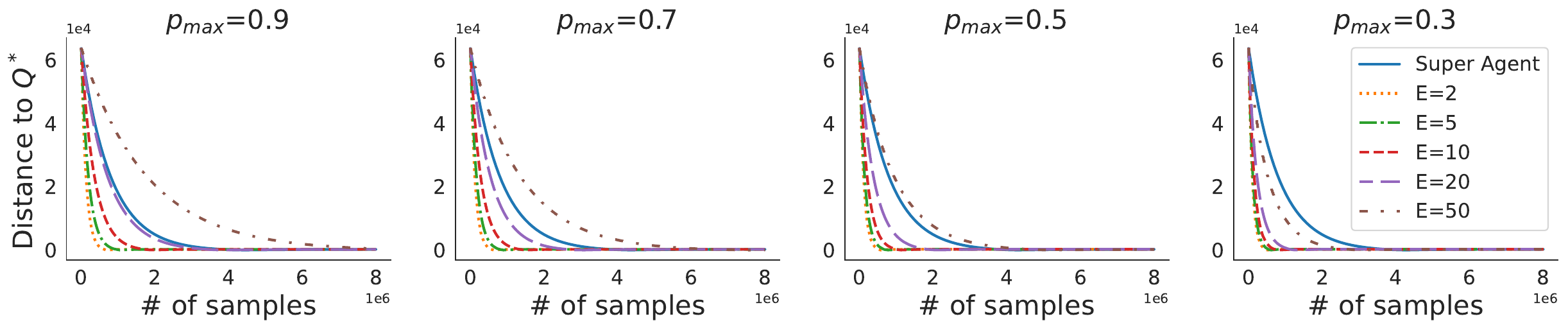}
    \includegraphics[width=0.8\textwidth]{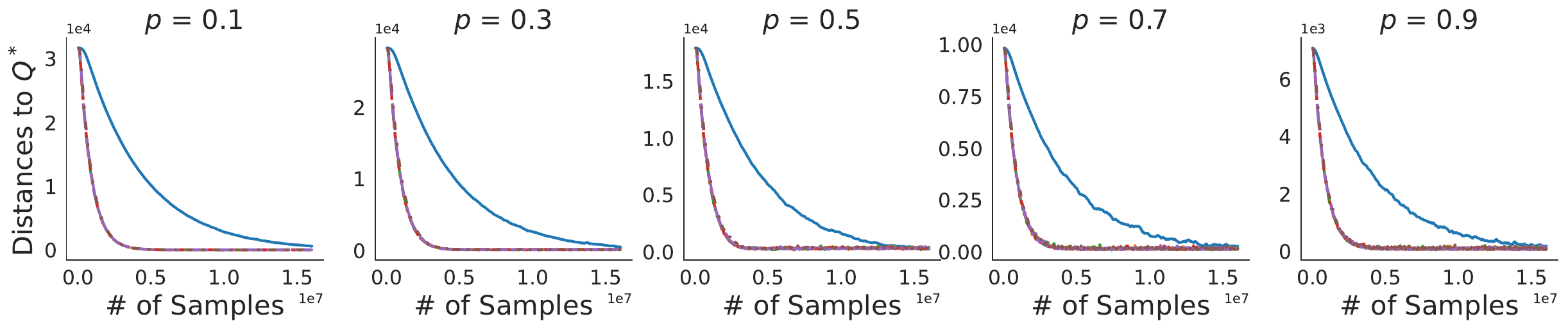}
    \caption{Convergence of FedQ-SynQ: the first row is in RandomMDP with different numbers of $N$, the second row is in RandomMDP with different  values of $p_{\max}$, and the third row is in WindyCliff with different values of wind power $p$.}
    \label{fig:FedQ-SyncQ}
\end{figure*}

\section{Numerical Experiments}
In this section, we conduct numerical experiments in tabular environments to examine the convergence of FedQ as well as sample complexity of FedQ-SynQ.

\subsection{The Set-up}
\textbf{Environments.} 
We construct two tabular environments with various settings of $\{S_k\}_{k=1}^N$ and transition dynamics $\PM$.
Specifically, RandomMDP with randomly generated transition dynamics, and WindyCliff with a navigation task.
Detailed description is left in Appendix \ref{App:env_construction}.

\textbf{Control of Heterogeneity.} 
In environments of RandomMDP, we generate instances with similar degrees of heterogeneity via controlling $(N,K_S,E_S,N_S)$ where 
\begin{align*}
    K_S = \#\{s:s\in S_k,N(s)=1\},&\forall k \in\{1,\dots,N\};\\
    E_S = \#\{s:s\in S_k,N(s)>1\},&\forall k \in\{1,\dots,N\};\\
    N(s)\in\{1,N_S\},&\forall s\in\SM
\end{align*}
where $K_S$ represents number of states an agent has unique access to, $E_S$ represents number of states it shares with other agents, and $N_S$ characterizes how many different agents are able to share access to a single state.
Additionally, we introduce $p_{\max}$ in some cases of RandomMDP to control the degree of connection among different regions
\begin{equation*}
    p_{\max}=\sum_{s'\notin \SM_k}\PM(s'|s,a),\forall (s,a)\in\SM_k\times\AM,\forall k\in\{1,\dots,N\}.
\end{equation*}
In environments of WindyCliff, the state space is split either horizontally or vertically.
Therefore, either actions of the agent or the unexpected wind results in transitions among different regions.
Both the splitting direction and the power of wind affect the degree of connection among $\{\SM_k\}_{k=1}^N$.

\textbf{Explanation on Baselines.}
To justify the speedup effect of FedQ-SynQ, we introduce single-agent synchronous Q-learning as baseline algorithms.
Specifically, a \textbf{super agent} is assumed to have access to the overall state space $\SM$, which means it generates samples for all state-action pairs every iteration.
Therefore, the super agent in fact has $N$-times larger sample complexity than a local agent with $|\SM_k|=|\SM|/N$ in a single iteration.
In this way, we compare sample complexity of the super agent and local agents in converging to the globally optimal Q function $Q^\star$.

\textbf{Other details.} We leave choices of hyper-parameters in constructing environments and the selection of learning rates in Appendix \ref{App:env_construction} and \ref{App:exper_hyper}.

\begin{figure}[!htb]
    \centering
    \vspace{-0.1in}
    \includegraphics[width=0.3\textwidth]{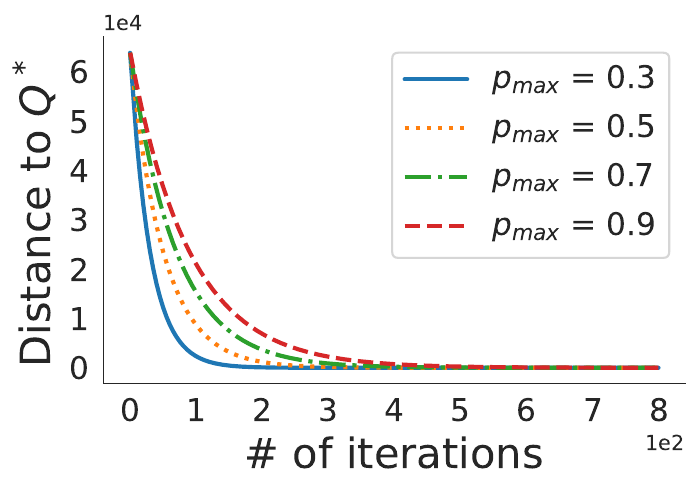}
    \caption{Iteration complexity of FedQ in different orthogonal cases of RandomMDP with $(N,K_S,E_S)=(10,20,0)$.}
    \vspace{-0.1in}
    \label{fig:pmax_exact}
\end{figure}

\subsection{Iteration Complexity of FedQ}
To justify the iteration complexity of FedQ, we conduct experiments in environments of RandomMDP with different settings of $p_{\max}$.
Specifically, these experiments require local agents to exactly solve the optimal policy in their local MDPs.
Therefore, Fig. \ref{fig:pmax_exact} unveils the relationship between iteration complexity and degrees of connection represented by $p_{\max}$.
Smaller $p_{\max}$ indicates more closed transition dynamics within local regions, which leads to a faster convergence of FedQ.

\subsection{Sample Complexity of FedQ-SynQ}
We have evaluated FedQ-SynQ with different choices of $E$ in RandomMDP and WindyCliff with different settings.
Specifically, Figure \ref{fig:FedQ-SyncQ} demonstrates results on three sets of environments: RandomMDP with different numbers of $N$, RandomMDP with different values of $p_{\max}$, and WindyCliff with different values of wind power.
Performance of FedQ-SynQ under more settings of these environments are left to Appendix \ref{App:exper_more}.
The first row reveals that a larger number of $N$ enlarges the speedup effect of FedQ-SynQ, the second row shows that FedQ-SynQ achieves faster convergence in a more closed dynamics with smaller $p_{\max}$, and the third row justifies the efficiency of FedQ-SynQ in WindyCliff with more realistic separation of state space.
\vspace{-0.1in}
\section{Conclusion}
We have studied problems of federated control with heterogeneity among state spaces accessible for different agents.
To quantify the influence of such heterogeneity on the learning process, concepts of leakage probabilities are introduced to depict degrees of connection among different regions.
For the collaboration learning of globally optimal policy, we propose FedQ as a simple but attractive communication protocol, and theoretically justify its correctness.
Moreover, we consider FedQ-X with any feasible RL oracle $X$, proves its linear speedup w.r.t. sample complexity in converging to optimal policies, and unveils relationship between its iteration complexity and leakage probabilities of federated control problems.

\bibliographystyle{plainnat}
\bibliography{refer}

\appendix
\newpage
\section{Properties of Federated Bellman Operator}\label{Appendix:FedOp}
\begin{lem}\label{lem:max}
For any two bounded real functions $f,~g$, we have the following inequality relationship on the absolute difference of their maximum:
    \begin{equation*}
        |\max_x f(x)-\max_x g(x)|\leq\max_x |f(x)-g(x)|.
    \end{equation*}
\end{lem}
\begin{proof}
    Without loss of generality, we assume $\max_x f(x)\geq\max_x g(x)$, we have
    \begin{equation*}
        \begin{aligned}
            |\max_x f(x)-\max_x g(x)|
            &=\max_x f(x)-\max_y g(y)\\
            &=\max_x \Big(f(x)-\max_y g(y)\Big)\\
            &\leq \max_x \Big(f(x)-g(x)\Big)\\
            &\leq \max_x \Big|f(x)-g(x)\Big|.
        \end{aligned}
    \end{equation*}
\end{proof}

\begin{lem}[Contraction property of $\Tfed^k$]\label{lem:Tfedk_contract}
    For all $k\in[N]$, the local Bellman operator $\Tfed^k$ satisfies the contraction property as follows:
    \begin{equation*}
        \max_{s\in\SM_k,a\in\AM}\Big|\Tfed^k(Q_1)(s,a)-\Tfed^k(Q_2)(s,a)\Big|\leq \gamma_\text{fed}^k\|Q_1-Q_2\|_\infty,
    \end{equation*}
    where $\gamma_\text{fed}^k=\frac{\gamma p^k_{\max}}{1-\gamma (1-p^k_{\max})}\leq \gamma$ increasing with $p^k_{\max}$, and $\frac{1}{1-\gamma_\text{fed}^k}=1+\frac{\gamma p_{\max}^k}{1-\gamma}\leq 1+\frac{p_{\max}^k}{1-\gamma}$.
\end{lem}
\begin{proof}
    The proof is trivial if $p^k_{\max}=0$.
    Hence, we only need to consider the case when $p^k_{\max}>0$, and in this scenario, $E_k\neq\emptyset$.
    By the definition of $\Tfed^k$, we have the following form of Bellman equation
    \begin{equation*}
        \Tfed^k(Q)(s,a)=\RM(s,a)+\gamma \EB_{\PM}\left [ \mathds{1}_k(s')\max_{a'}\Tfed^k(Q)(s',a')+(1-\mathds{1}_k(s'))\max_{a'} Q(s',a')\right ], \forall (s,a)\in\SM_k\times\AM.
    \end{equation*}
    For simplicity, we abbreviate $\Tfed^k(Q_1),\Tfed^k(Q_2)$ as $\tilde{Q}_1,\tilde{Q}_2$, and $\|\cdot\|_\infty$ as $\|\cdot\|$.
    And we define $\|Q\|_{\SM_k}:=\max_{s\in\SM_k,a\in\AM}\Big|Q(s,a)\Big|$.
    For all $(s,a)\in\SM_k\times\AM$,
    \begin{equation*}
            \begin{aligned}
    &~~\Big|\tilde{Q}_1(s,a)-\tilde{Q}_2(s,a)\Big|\\
    =&\gamma\Bigg | \EB_{\PM}\Big [\mathds{1}_k(s')\Big (\max_{a'}\tilde{Q}_1(s',a')-\max_{a'}\tilde{Q}_2(s',a')\Big )+(1-\mathds{1}_k(s'))\Big (\max_{a'} Q_1(s',a')-\max_{a'} Q_2(s',a')\Big )\Big]\Bigg | \\
    =&\gamma\Bigg |\sum_{s'\in\SM_k}\PM(s'|s,a)\Big (\max_{a'}\tilde{Q}_1(s',a')-\max_{a'}\tilde{Q}_2(s',a')\Big )+\sum_{s'\notin \SM_k}\PM(s'|s,a)\Big (\max_{a'} Q_1(s',a')-\max_{a'} Q_2(s',a')\Big )\Bigg |\\
    \leq& \gamma \Big|\sum_{s'\in\SM_k}\PM(s'|s,a)\Big (\max_{a'}\tilde{Q}_1(s',a')-\max_{a'}\tilde{Q}_2(s',a')\Big )\Big|+\gamma \Big|\sum_{s'\notin \SM_k}\PM(s'|s,a)\Big (\max_{a'} Q_1(s',a')-\max_{a'} Q_2(s',a')\Big) \Big|\\
    \leq& \gamma \Big (\sum_{s'\in\SM_k}\PM(s'|s,a)\Big )\max_{s'\in\SM_k}\Big|\max_{a'}\tilde{Q}_1(s',a')-\max_{a'}\tilde{Q}_2(s',a')\Big|+\gamma \Big (\sum_{s'\notin\SM_k}\PM(s'|s,a)\Big )\max_{s'\notin\SM_k}\Big|\max_{a'}Q_1(s',a')-\max_{a'}Q_2(s',a')\Big|\\
    \leq&\gamma \Big (\sum_{s'\in\SM_k}\PM(s'|s,a)\Big )\max_{s'\in\SM_k,a'\in\AM}\Big|\tilde{Q}_1(s',a')-\tilde{Q}_2(s',a')\Big|+\gamma \Big (\sum_{s'\notin\SM_k}\PM(s'|s,a)\Big )\max_{s'\notin\SM_k,a'\in\AM}\Big|Q_1(s',a')-Q_2(s',a')\Big|\\
    \leq&\gamma(1-\PM_k(s,a))\|\tilde{Q}_1-\tilde{Q}_2\|_{\SM_k}+\gamma\PM_k(s,a)\|Q_1-Q_2\|,
        \end{aligned}
    \end{equation*}

    For $(s,a)\in K_k\times\AM$, we get $\PM_k(s,a)=0$ by the definition of $K_k$.
    Therefore, we are able to derive
        \begin{equation*}
        \Big|\tilde{Q}_1(s,a)-\tilde{Q}_2(s,a)\Big|\leq \gamma\|\tilde{Q}_1-\tilde{Q}_2\|_{\SM_k},\forall(s,a)\in K_k\times\AM,
    \end{equation*}
    which means the maximum absolute difference between $\tilde{Q}_1$ and $\tilde{Q}_2$ does not take place within $K_k\times\AM$.

    For $(s,a)\in E_k\times\AM$, we get $\PM_k(s,a)\in[p^k_{\min},p^k_{\max}]\subset (0,1]$, where $p_{\min}^k:=\min_{(s,a)\in E_k}\PM_k(s,a)>0$
, hence
        \begin{equation*}
        \Big|\tilde{Q}_1(s,a)-\tilde{Q}_2(s,a)\Big|\leq \gamma(1-p^k_\star)\|\tilde{Q}_1-\tilde{Q}_2\|_{\SM_k}+\gamma p^k_\star\|Q_1-Q_2\|,\forall(s,a)\in E_k\times\AM,
    \end{equation*}
    where $p^k_\star=p^k_{\min}\mathds{1}\{\|Q_1-Q_2\|< \|\tilde{Q}_1-\tilde{Q}_2\|_{\SM_k} \}+p^k_{\max}\mathds{1}\{ \|Q_1-Q_2\|\geq \|\tilde{Q}_1-\tilde{Q}_2\|_{\SM_k}\}\in(0,1]$. 
    Take the supremum on both sides and we get
    \begin{equation*}
            \begin{aligned}
    \|\tilde{Q}_1-\tilde{Q}_2\|_{\SM_k} \leq \gamma(1-p^k_\star)\|\tilde{Q}_1-\tilde{Q}_2\|_{\SM_k}+\gamma p^k_\star\|Q_1-Q_2\|,
    \end{aligned}
    \end{equation*}
    therefore,
    \begin{equation*}
    \begin{aligned}
    \|\tilde{Q}_1-\tilde{Q}_2\|_{\SM_k} \leq \frac{\gamma p^k_\star}{1-\gamma (1-p^k_\star)}\|Q_1-Q_2\|:=\gamma_\text{fed}^k\|Q_1-Q_2\|,
    \end{aligned}
    \end{equation*}
    where
    \begin{equation}
        \begin{aligned}
                    \frac{1}{\gamma_\text{fed}^k}&=\frac{1-\gamma(1-p^k_\star)}{\gamma p^k_\star}\\
                    &=\frac{1-\gamma}{\gamma p^k_\star}+1\\
                    &\in [\frac{1}{\gamma},+\infty),
        \end{aligned}
    \end{equation}
    hence $\gamma_{\text{fed}}^k\in (0,\gamma]$, $p^k_\star=p^k_{\max}$ and $\gamma_\text{fed}^k$ increases with $p^k_{\max}$.

    To summarize, we have proved that local Bellman operator $\Tfed^k$ satisfies the contraction property with $\gamma_\text{fed}^k=\frac{\gamma p^k_{\max}}{1-\gamma (1-p^k_{\max})}\in (0,\gamma]$.
\end{proof}

\begin{proof}[Proof of Lemma~\ref{lem:Tfed_contraction}]
    For any $(s,a)\in\SM\times\AM$ and any $Q_1,Q_2\in\RB^{|\SM||\AM|}$,
    \begin{align}
        \notag
        &|\Tfed(Q_1)(s,a)-\Tfed(Q_2)(s,a)|\\
        \notag
        =&\frac{1}{N(s)}\bigg |\sum_{k=1}^N\big (\Tfed^k(Q_1)-\Tfed^k(Q_2)\big )(s,a)\mathds{1}_k(s)\bigg|\\
        \label{ana:Tfed_max}
        \leq&\frac{1}{N(s)}\sum_{k=1}^N\bigg |\big (\Tfed^k(Q_1)-\Tfed^k(Q_2)\big )(s,a)\mathds{1}_k(s)\bigg|\\
        \label{ana:Tfed_contract}
        \leq&\frac{1}{N(s)}\sum_{k=1}^N \gamma_\text{fed}^k\|Q_1-Q_2\|_\infty\mathds{1}_k(s)\\
        \label{ana:Tfed_gamma}
        \leq&\gamma_\text{fed}\|Q_1-Q_2\|_\infty,
    \end{align}
    where (\ref{ana:Tfed_max}) holds because of the triangle inequality, (\ref{ana:Tfed_contract}) comes from contraction properties of $\{\Tfed^k\}_{k=1}^N$, and (\ref{ana:Tfed_gamma}) naturally holds because of the definition of $\gamma_\text{fed}$.
\end{proof}

\begin{proof}[Proof of Lemma~\ref{lem:Tfed_fixpoint}]
        First, we introduce the following one step local optimistic Bellman operator $\Tfedn{1}^{k,\tilde{V}}\colon\RB^{|\SM||\AM|}\to\RB^{|\SM||\AM|}$, which is the standard optimistic Bellman operator for $\MM_k$ (note that we do not need to deal with $s_{null}$ here and we ignore the input $Q(s,a)$ for $s\notin\SM_k$), \ie, for any $a\in\AM, s\in\SM_k$,
        \begin{equation*}
                    \Tfedn{1}^{k,\tilde{V}}(Q)(s,a)=\RM(s,a)+\gamma\EB_{\PM}\Big [\mathds{1}_k(s')\max_{a'}Q(s',a')+(1-\mathds{1}_k(s'))\tilde{V}(s')\Big ],
        \end{equation*}
    for any $s\notin\SM_k$, $\Tfedn{1}^{k,\tilde{V}}(Q)(s,a)=\tilde{V}(s)$. By definition, $\Tfed^k(Q^\star)$ is the unique fixed point of $\Tfedn{1}^{k,V^\star}$. We define \begin{equation*}
    Q_k^\star(s,a)=
    \begin{cases}
    Q^\star(s,a)&s\in\SM_k\\
    V^\star(s)=\max_{a'}Q^\star(s,a')&s\notin\SM_k\\
    \end{cases},
\end{equation*}
which is equal to $Q^\star(s,a)$ when $s\in\SM_k$.
    Let's check $Q^*_k$ is that fixed point, for any $a\in\AM, s\in\SM_k$,
        \begin{align}
        \notag
        &\Tfedn{1}^{k,V^\star}(Q^\star_k)(s,a)\\
        \notag
        =&\RM(s,a)+\gamma\EB_{\PM}\Big [\mathds{1}_k(s')\max_{a'}Q^\star(s',a')+(1-\mathds{1}_k(s'))V^\star(s')\Big ]\\
        \notag
        =&\RM(s,a)+ \gamma\EB_{\PM}\Big [\max_{a'}Q^{\star}(s',a')\Big ]\\
        \notag
        =&Q^\star(s,a),
    \end{align}
for any $s\notin\SM_k$, $\Tfedn{1}^{k,V^\star}(Q^\star_k)(s,a)=V^\star(s)$. Therefore, $\Tfed^k(Q^\star)=Q^\star_k$, we have
    \begin{align}
        \notag
        &\Tfed(Q^\star)(s,a)\\
        \notag
        =&\frac{1}{N(s)}\sum_{k=1}^N \Tfed^{k}(Q^\star)(s,a)\mathds{1}_k(s)\\
        \notag
        =&\frac{1}{N(s)}\sum_{k=1}^N Q^\star_k(s,a)\mathds{1}_k(s)\\
        \notag
        =&\frac{1}{N(s)}\sum_{k=1}^N Q^\star(s,a)\mathds{1}_k(s)\\
        \notag
        =&Q^\star(s,a),
    \end{align}
which is desired.
\end{proof}

\section{Sample Complexity of FedQ-X}\label{Appendix:SampleComplexityFedQX}
For simplicity, in this section, we will abbreviate $\|\cdot\|_\infty$ as $\|\cdot\|$ and denote $\TM_{\text{fed},r}(Q_{r-1}):=Q_{r}$ and $\TM_{\text{fed},r}^k(Q_{r-1}):=Q_{r-1,k}$ for all $r\in[R]$.
\begin{proof}[Proof of Lemma~\ref{lem:converge_fedqx}]
We need to bound $\|Q^\star-Q_R\|$:
    \begin{equation*} 
        \begin{aligned}
            &\|Q^\star-Q_R\|\\
            =& \|\Tfed(Q^\star)-\TM_{\text{fed},R}(Q_{R-1})\|\\
            =& \left\| \left[\frac{1}{N(s)}\sum_{k=1}^N \left(\Tfed^k(Q^\star)(s,a)-\TM_{\text{fed},R}^k(Q_{R-1})(s,a)\right)\mathds{1}_k(s)\right]_{(s,a)\in\SM\times\AM} \right\|\\
            \leq& \max_{k\in [N]}\|\Tfed^k(Q^\star)-\TM_{\text{fed},R}^k(Q_{R-1}) \|_{\SM_k}\\
            \leq&\max_{k\in [N]}\|\Tfed^k(Q^\star)-\Tfed^k(Q_{R-1})\|_{\SM_k}+ \max_{k\in [N]}\|\Tfed^k(Q_{R-1})-\TM_{\text{fed},R}^k(Q_{R-1}) \|_{\SM_k}\\
            \leq& \gamma_\text{fed}\|Q^\star-Q_{R-1} \|+ \varepsilon_R,
        \end{aligned}
    \end{equation*}
    apply the upper bound recursively,
    \begin{equation*} 
        \begin{aligned}
            &\|Q^\star-Q_R\|\\
            \leq &\gamma_\text{fed}^R\|Q^\star-Q_0 \|+ \sum_{r=1}^R\gamma_\text{fed}^{R-r}\varepsilon_{r}\\
            \leq &\frac{\gamma_\text{fed}^R}{1-\gamma}+\sum_{r=1}^R\gamma_\text{fed}^{R-r}\varepsilon_r,
        \end{aligned}
\end{equation*}
    hence $\limsup_{R\to\infty}\|Q^\star-Q_R\|\leq\limsup_{R\to\infty}\sum_{r=1}^R\gamma_\text{fed}^{R-r}\varepsilon_r$.
    We assume $\limsup_{R\to\infty}\varepsilon_R=C<\infty$, otherwise $\limsup_{R\to\infty}\sum_{r=1}^R\gamma_\text{fed}^{R-r}\varepsilon_r\leq \limsup_{R\to\infty} \varepsilon_R=\infty$.
    Then for any $\epsilon>0$, there exists an $r(\epsilon)>0$ such that $\sup_{r\geq r(\epsilon)}\varepsilon_r\leq C+\epsilon$, hence
\begin{equation*} 
        \begin{aligned}
\limsup_{R\to\infty}\sum_{r=1}^R\gamma_\text{fed}^{R-r}\varepsilon_r&=\limsup_{R\to\infty}\sum_{r=1}^{r(\epsilon)-1}\gamma_\text{fed}^{R-r}\varepsilon_r+\limsup_{R\to\infty}\sum_{r=r(\epsilon)}^{R}\gamma_\text{fed}^{R-r}\varepsilon_r\\
&\leq (C+\epsilon)\limsup_{R\to\infty}\sum_{r=r(\epsilon)}^{R}\gamma_\text{fed}^{R-r}\\
&=\frac{C+\epsilon}{1-\gamma_\text{fed}}\limsup_{R\to\infty}(1-\gamma_\text{fed}^{R-r(\epsilon)+1})\\
&=\frac{C+\epsilon}{1-\gamma_\text{fed}},
        \end{aligned}
\end{equation*}
which holds for any $\epsilon>0$, therefore $\limsup_{R\to\infty}\sum_{r=1}^R\gamma_\text{fed}^{R-r}\varepsilon_r\leq\frac{1}{1-\gamma_\text{fed}}\limsup_{R\to\infty}\varepsilon_R$.
\end{proof}

\begin{proof}[Proof of Theorem~\ref{Thm:SampleComplexity_FedQX}]
Recall the upper bound we have derived:
    \begin{equation*} 
        \begin{aligned}
            \|Q^\star-Q_R\|
            \leq \frac{\gamma_\text{fed}^R}{1-\gamma}+\sum_{r=1}^R\gamma_\text{fed}^{R-r}\max_{k\in [N]}\|\Tfed^k(Q_{R-1})-\TM_{\text{fed},R}^k(Q_{R-1}) \|_{\SM_k},
        \end{aligned}
\end{equation*}
here we can bound the second term using the theoretical guarantee of the RL algorithm X. 
To be concrete, for all $r\in[R]$, if we use $n_{r,k}:= n\left(|\SM_k|,|\AM|,\gamma,\varepsilon_r,\delta_r\right)$ samples, we have with probability at least $1-\delta_r$, $\max_{k\in [N]}\|\Tfed^k(Q_{r-1})-\TM_{\text{fed},r}^k(Q_{r-1}) \|_{\SM_k}\leq\varepsilon_r$ for some $\delta_r>0$ and $\varepsilon_r>0$.
Apply the union bound, we obtain  with probability at least $1-\sum_{r=1}^R\delta_r$,
    \begin{equation*} 
        \begin{aligned}
            \|Q^\star-Q_R\|
            \leq \frac{\gamma_\text{fed}^R}{1-\gamma}+\sum_{r=1}^R\gamma_\text{fed}^{R-r}\varepsilon_r.
        \end{aligned}
\end{equation*}

For any $\varepsilon\in(0,\frac{1}{1-\gamma})$, to make sure the first term not larger than $\frac{\varepsilon}{2}$, we should take $R\geq \log\left(\frac{2}{\varepsilon(1-\gamma)}\right)/\log\left(\frac{1}{\gamma_\text{fed}}\right)=\widetilde{\OM}\left(\frac{1}{1-\gamma_\text{fed}}\right)$, where $1/\log\left(\frac{1}{\gamma_\text{fed}}\right)=-1/\log\left(1-(1-\gamma_\text{fed})\right)\leq \frac{1}{1-\gamma_\text{fed}}$.

We can take $\varepsilon_r$ as constant: when $\varepsilon_r=\frac{\varepsilon}{2}(1-\gamma_\text{fed})$, we have the second term $\sum_{r=1}^R\gamma_\text{fed}^{R-r}\varepsilon_r=\frac{\varepsilon}{2}(1-\gamma_\text{fed})\sum_{r=1}^R \gamma_\text{fed}^{R-r}<\frac{\varepsilon}{2}$, resulting in $\|Q^\star-Q_R\|\leq \varepsilon$.

For any $\delta\in(0,1)$, we can take $\delta_r=\frac{C\delta}{r[\log(r)]^2}$ which does not depend on $R$, and $C:=\left(\sum_{r=1}^\infty \frac{1}{t[\log(t)]^2}\right)^{-1}$. Hence, $\sum_{r=1}^R\delta_r<\delta\sum_{r=1}^\infty\frac{C}{r[\log(r)]^2}=\delta$ and $\delta_r=\widetilde{\Omega}\left(\delta(1-\gamma_\text{fed})\right)$ for all $r\in[R]$.

To summarize, the total sample complexity at device $k$ is $n_k=\sum_{r=1}^R n_{r,k}$.
, we have
\begin{equation*} 
        \begin{aligned}
n_k&=\sum_{r=1}^R n\left(|\SM_k|,|\AM|,\gamma,\varepsilon_r,\delta_r\right)\\
&=\widetilde{O}\left(\frac{|\SM_k||\AM|}{\varepsilon^2(1-\gamma_\text{fed})^3(1-\gamma)^m}  \right)\\
&=\widetilde{O}\left(\frac{|\SM_k||\AM|}{\varepsilon^2(1-\gamma)^m}\left(1+\frac{p_{\max}^3}{(1-\gamma)^3}\right)  \right),
\end{aligned}
\end{equation*}
if $p_{\max}>0$, which is typically the case, we can determine that $n_k=\widetilde{O}\left(\frac{p_{\max}^3|\SM_k||\AM|}{\varepsilon^2(1-\gamma)^{m+3}}\right)$.
\end{proof}
\textit{\textbf{Remark}} 5: The sample complexity of solving the $k$-th local MDP may not be precisely given by $n\left(|\SM_k|,|\AM|,\gamma,\varepsilon_t,\delta_t\right)$. 
It may include a logarithmic term with respect to $|\SM|$ instead of $|\SM_k|$ for common provable RL solvers.
However, we can neglect this small difference since we do not consider logarithmic terms in this paper.
\section{Sample Complexity of FedQ-SynQ}\label{section:proof_fedq_synq}
Recall the one step local optimistic
Bellman operator $\Tfedn{1}^{k,\tilde{V}}\colon\RB^{|\SM||\AM|}\to\RB^{|\SM||\AM|}$ for a given $\tilde{V}\in\RB^{|\SM|}$, which is the standard optimistic Bellman operator for $\MM_k$ (we do not need to deal with $s_{null}$ here), \ie, for any $a\in\AM, s\in\SM_k$,
\begin{equation*}
    \Tfedn{1}^{k,\tilde{V}}(Q)(s,a)=\RM(s,a)+\gamma\EB_{\PM}\Big [\mathds{1}_k(s')\max_{a'}Q(s',a')+(1-\mathds{1}_k(s'))\tilde{V}(s')\Big ],
\end{equation*}
    for any $s\notin\SM_k$, $\Tfedn{1}^{k,\tilde{V}}(Q)(s,a)=\tilde{V}(s)$.

We denote the $n$ step standard optimistic Bellman operator for $\MM_k$ as $\Tfedn{n}^{k,\tilde{V}}:=\Big[\Tfedn{1}^{k,\tilde{V}}\Big]^n$ and define the $n$ step local optimistic
Bellman operator (we abuse the name here, since this operator does not depend on a given $\tilde{V}$), $\Tfedn{n}^{k}\colon\RB^{|\SM||\AM|}\to\RB^{|\SM||\AM|}$, $\Tfedn{n}^{k}(Q):=\Tfedn{n}^{k,\max_{a\in\AM}Q(\cdot,a)}(Q)$, for all $Q\in\RB^{|\SM||\AM|}$. We will show that $\Tfedn{n}^{k}$ also satisfies the contraction property.

\begin{lem}
[Contraction property of $\Tfedn{n}^{k}$]\label{lem:Tfedkn_contract}
For any $n\geq 1$, $\Tfedn{n}^{k}$ satisfies the contraction property as follows:
    \begin{equation*}
                \max_{s\in\SM_k,a\in\AM}\Big|\Tfedn{n}^{k}(Q_1)(s,a)-\Tfedn{n}^{k}(Q_2)(s,a)\Big|\leq\gamma_{\text{fed-}n}^k\|Q_1-Q_2\|_\infty,
        ~\forall Q_1,Q_2\in\RB^{|\SM||\AM|},
    \end{equation*}
where $\gamma_{\text{fed-}n}^k=\gamma\Big\{[\gamma(1-p^k_{\max})]^{n-1}+  p^k_{\max}\frac{1-[\gamma(1-p^k_{\max})]^{n-1}}{1-\gamma(1-p^k_{\max})}\Big\}\in[\gamma_{\text{fed}}^k,\gamma]$, decreasing with $n$, and $\frac{1}{1-\gamma_{\text{fed-}n}}=\frac{1}{1-\gamma}\frac{1-\gamma(1-p^k_{\max})}{1-[\gamma(1-p^k_{\max})]^{n}}$.
\end{lem}
\begin{proof}
    The proof is trivial if $p^k_{\max}=0$, we only need to consider the case when $p^k_{\max}>0$.
    Following the proof of Lemma~\ref{lem:Tfedk_contract}, we can similarly obtain
    \begin{equation*}
        \|\Tfedn{1}^{k}(Q_1)-\Tfedn{1}^{k}(Q_2)\|_{\SM_k}\leq \gamma\|Q_1-Q_2\|, 
    \end{equation*}
    and we can prove by induction that $\|\Tfedn{n}^{k}(Q_1)-\Tfedn{n}^{k}(Q_2)\|_{\SM_k}\leq \gamma\|Q_1-Q_2\|$.
    In fact, we can derive a slight better result: for $n\geq 2$,
        \begin{equation*}
        \|\Tfedn{n}^{k}(Q_1)-\Tfedn{n}^{k}(Q_2)\|_{\SM_k}\leq \gamma(1-p^k_{\max})\|\Tfedn{(n-1)}^{k}(Q_1)-\Tfedn{(n-1)}^{k}(Q_2)\|_{\SM_k}+\gamma p^k_{\max}\|Q_1-Q_2\|,
    \end{equation*}
    hence,
    \begin{equation*}
        \|\Tfedn{n}^{k}(Q_1)-\Tfedn{n}^{k}(Q_2)\|_{\SM_k}\leq \gamma\Big\{[\gamma(1-p^k_{\max})]^{n-1}+  p^k_{\max}\frac{1-[\gamma(1-p^k_{\max})]^{n-1}}{1-\gamma(1-p^k_{\max})}\Big\}\|Q_1-Q_2\|,
    \end{equation*}
    where
    \begin{equation*}
        \begin{aligned}
            \gamma_{\text{fed-}n}^k&:=\gamma\Big\{[\gamma(1-p^k_{\max})]^{n-1}+  p^k_{\max}\frac{1-[\gamma(1-p^k_{\max})]^{n-1}}{1-\gamma(1-p^k_{\max})}\Big\}\\
            &=\gamma\Big\{\frac{p^k_{\max}}{1-\gamma(1-p^k_{\max})}+  \Big(1-\frac{p^k_{\max}}{1-\gamma(1-p^k_{\max})}\Big)[\gamma(1-p^k_{\max})]^{n-1}\Big\},
        \end{aligned}
    \end{equation*}
    which is decreasing with $n$, since the coefficient $1-\frac{p^k_{\max}}{1-\gamma(1-p^k_{\max})}\geq 0$ and the base $\gamma(1-p^k_{\max})\in (0,1)$. 
    And we can check that $\gamma_{\text{fed-}1}^k=\gamma$ and $\lim_{n\to\infty}\gamma_{\text{fed-}n}^k=\gamma_{\text{fed}}^k$.
    \begin{equation*}
        \begin{aligned}
            \frac{1}{1-\gamma_{\text{fed-}n}^k}&=\frac{1}{1-\gamma\Big\{[\gamma(1-p^k_{\max})]^{n-1}+  p^k_{\max}\frac{1-[\gamma(1-p^k_{\max})]^{n-1}}{1-\gamma(1-p^k_{\max})}\Big\}}  \\
            &=\frac{1}{1-\gamma}\frac{1-\gamma(1-p^k_{\max})}{1-[\gamma(1-p^k_{\max})]^{n}}.
        \end{aligned}
    \end{equation*}

\end{proof}
Based on Lemma~\ref{lem:Tfedkn_contract}, we can regard $\Tfed^k$ as $\Tfedn{\infty}^k$. 
And we have the following general version of Lemma~\ref{lem:Tfed_contraction}.
\begin{lem}[Contraction property of $\TM_{\text{fed-}\{n_k\}_{k=1}^N}$]
Suppose $n_k\in\NB\cup\{\infty\}$ for $k\in[N]$, we define the corresponding federated Bellman operator:
\begin{equation*}
    \TM_{\text{fed-}\{n_k\}_{k=1}^N}(Q)(s,a)=\frac{1}{N(s)}\sum_{k=1}^N\Tfedn{n_k}^k(Q)(s,a)\mathds{1}_k(s),
\end{equation*}
then $\TM_{\text{fed-}\{n_k\}_{k=1}^N}$ also satisfies the contraction property as follows
\begin{equation*}
    \|\TM_{\text{fed-}\{n_k\}_{k=1}^N}(Q_1)-\TM_{\text{fed-}\{n_k\}_{k=1}^N}(Q_2)\|_\infty\leq\Big(\max_{k\in[N]}\gamma_{\text{fed-}n_k}^k\Big)\|Q_1-Q_2\|_\infty, \forall Q_1,Q_2\in\RB^{|\SM||\AM|}.
\end{equation*}
\end{lem}
The proof is straightforward, and we omit it.

We have the following general version of Lemma~\ref{lem:Tfed_fixpoint} for this federated Bellman operator:
\begin{lem}
    [Stationary point of $\TM_{\text{fed-}\{n_k\}_{k=1}^N}$] 
    $Q^\star$ is the unique stationary point of the contraction mapping $\TM_{\text{fed-}\{n_k\}_{k=1}^N}$.
\end{lem}
The proof is straightforward, as long as we note that $\Tfedn{n}^{k}(Q^\star)=\Tfedn{n}^{k,V^\star}(Q^\star)=\Big[\Tfedn{1}^{k,V^\star}\Big]^n(Q^\star)=Q_k^\star$, where $Q_k^\star$ is defined in the proof of Lemma~\ref{lem:Tfed_fixpoint}.

Inspired by the above analysis, we can now perform a separate analysis on the sample complexity of FedQ-SynQ, which employs synchronous Q-Learning as the RL algorithm X.

\begin{thm}[Sample Complexity of FedQ-SynQ]
Consider any given $\delta \in (0,1)$ and $\varepsilon\in (0, \frac{1}{1-\gamma})$.
Suppose the initialization of the global Q function obeys $0\leq Q_0(s,a)\leq\frac{1}{1-\gamma},\forall (s,a)\in\SM\times\AM$,  
the total update steps $T=ER$ satisfies
\begin{equation*}
    T \geq  \frac{1296}{N_{\min}\varepsilon^2(1-\gamma)^5 } \left(\log\frac{10}{\varepsilon(1-\gamma)^2}\right)^2 \log{\frac{6|\SM||\AM|TN}{\delta}},
\end{equation*}
\ie, $T=\widetilde{O}\left(\frac{1}{N_{\min}\varepsilon^2(1-\gamma)^5 }\right)$, the $k$-th agent needs $\widetilde{O}\left(\frac{|\SM_k||\AM|}{N_{\min}\varepsilon^2(1-\gamma)^5 }\right)$ samples, the step size $\eta$ satisfies
\begin{equation*}
    \eta \leq \frac{1}{1296} N_{\min} \varepsilon^2 (1-\gamma)^4  \left(\log{\frac{6|\SM||\AM|TN}{\delta}}\right)^{-1},
\end{equation*}
\ie, $\eta=\widetilde{\Omega}\left(N_{\min} \varepsilon^2 (1-\gamma)^4\right)$,
and the local update steps per communication round $E$ satisfies
\begin{equation*}
    E \le  1+ \frac{1}{1.01\eta} \min \left\{  \frac{1-\gamma}{4\gamma},  \frac{1}{N_{\min}} \right\}, 
\end{equation*}
\ie, $E=\widetilde{\Omega}\left(\frac{1}{\eta}\min \left\{  1-\gamma, \frac{1}{N_{\min}} \right\}\right)$,
where $N_{\min}:=\min_{s\in\SM}N(s)$.
Then, with probability at least $1-\delta$, the final estimate $Q_R$ is $\varepsilon$-optimal, \ie, $\|Q^{\star} - Q_R\|_{\infty}\leq \varepsilon$.

Furthermore, if we take $T=\widetilde{\Theta}\left(\frac{1}{N_{\min}\varepsilon^2(1-\gamma)^5 }\right)$ and $\eta=\tilde{\Theta}(N_{\min} \varepsilon^2 (1-\gamma)^4)$, then $E=\widetilde{\Omega}\left(\frac{1}{N_{\min} \varepsilon^2 (1-\gamma)^4}\min \left\{  1-\gamma, \frac{1}{N_{\min}} \right\}\right)$ and the total communication round satisfies $R=\widetilde{O}\left(\frac{1}{1-\gamma}\max \left\{ \frac{1}{1-\gamma}, N_{\min} \right\} \right)$. 
\end{thm}
\begin{proof}
The proof of this theorem follows a similar approach to the proof of Theorem 1 in \cite{woo2023blessing}.
However, we need to be cautious in addressing the differences in the analysis arising from the variations in the update scheme of FedQ-SynQ.
\paragraph{Notations and Update Schemes:}
Let us introduce some notations in the following proof. We will slightly abuse the notations, the subscript of $Q$ (and some related variables) represents the $t$-th iteration, rather than the $r$-th communication round. 
We denote $\iota(t): =E \lfloor \frac{t}{E} \rfloor $ as the most recent synchronization step until $t$.

The Q function of agent $k\in[N]$ is initialized as $Q_{0,k}=Q_0$, it is worthy noting that the $k$-th agent only updates the Q function for $s\in \SM_k$ locally. 
The local update scheme for the local Q function without aggregation is given by: for all $(s,a)\in\SM_k\times\AM$, $t\in[T]$, 
\begin{equation*}
\begin{aligned}
        Q_{t-\frac{1}{2},k}(s, a) &:= (1-\eta) Q_{t-1,k}(s,a) + \eta \left[ \RM(s,a) + \gamma\left( \ind_k(s_{t,k}(s,a))  V_{t-1,k}(s_{t,k}(s,a)) +\ind_k^c(s_{t,k}(s,a))  V_{\iota(t-1)}(s_{t,k}(s,a)) \right)  \right]\\
        &=(1-\eta) Q_{t-1,k}(s,a) + \eta \left( \RM(s,a) + \gamma V_{t-1,k}(s_{t,k}(s,a))\right),
\end{aligned}
\end{equation*}
where $s_{t,k}(s,a)\sim \PM(\cdot|s,a)$ independently, $\ind_k^c(s):=1-\ind_k(s)$, and
\begin{equation*}
    V_{t,k}(s) =  \max_{a\in \AM} Q_{t,k}(s,a),\forall s\in\SM,
\end{equation*}
\begin{equation*}
    V_{t}(s) :=  \max_{a\in \AM} Q_{t}(s,a),\forall s\in\SM,
\end{equation*}
\begin{equation*}
    Q_{t}(s) := \frac{1}{N(s)}\sum_{k=1}^N Q_{t-\frac{1}{2},k}(s,a)\ind_k(s),\forall (s,a)\in\SM\times\AM,
\end{equation*}
the second equality holds because $V_{t-1,k}(s)=V_{\iota(t-1)}(s)$ for $s\notin\SM_k$.
Note that the definition of $Q_t$ does not require $t \equiv 0 \; (\text{mod}~ E)$.

The local Q function is aggregated as follow: for all $(s,a)\in\SM\times\AM$,
\begin{equation*}
    \begin{aligned}
Q_{t,k}(s,a) =
  \begin{cases}
    \frac{1}{N(s)} \sum_{k=1}^N Q_{t-\frac{1}{2},k}(s,a)\ind_k(s) &~\text{if}~ t \equiv 0 \; (\text{mod}~ E)\\
    Q_{t-\frac{1}{2},k}(s,a) &~\text{otherwise}
  \end{cases}.
\end{aligned}
\end{equation*}

We further define the transition matrix $P:=\left(\PM(s'|s,a) \right)_{(s,a)\in\SM\times\AM,s'\in\SM} \in \RB^{|\SM||\AM| \times |\SM|}$, the local empirical transition matrix at the $t$-th iteration $P_{t,k}:=\left(\ind\{s'=s_{t,k}(s,a) \}\right)_{(s,a)\in\SM\times\AM,s'\in\SM} \in \{0,1\}^{|\SM||\AM| \times |\SM|}$, error $\Delta_t:=Q^\star-Q_t\in\RB^{|\SM||\AM|}$, projection matrix for $k$-th agent $\Pi_k:=\text{diag}\left\{ \left(\ind\{s\in\SM_k \}\right)_{(s,a)\in\SM\times\AM}\right\}\in\RB^{|\SM||\AM|\times|\SM||\AM|}$, and weighting matrix $\Lambda:=\text{diag}\left\{ \left(\frac{1}{N(s)}\right)_{(s,a)\in\SM\times\AM}\right\}\in\RB^{|\SM||\AM|\times|\SM||\AM|}$.
We can check that,
\begin{equation*}
    \Pi_k(\IM_{|\SM||\AM|}-\Pi_k)=0,
\end{equation*}
\begin{equation*}
    \Lambda\sum_{k=1^N}\Pi_k=\IM_{|\SM||\AM|},
\end{equation*}
where $\IM_{|\SM||\AM|}\in \RB^{|\SM||\AM|\times|\SM||\AM|}$ is the identity matrix, we write $\Pi_k^c:=\IM_{|\SM||\AM|}-\Pi_k$
With these notations, we can express the update scheme in a matrix-vector form as follows:
\begin{equation*}
        Q_{t-\frac{1}{2},k}=\Pi_k\left[(1-\eta) Q_{t-1,k} + \eta \left( \RM + \gamma P_{t,k} V_{t-1,k}\right)\right]+\Pi_k^c Q_{t-1,k},
\end{equation*}
\begin{equation*}
        Q_t=\Lambda\sum_{k=1}^N \Pi_k Q_{t,k}=\Lambda\sum_{k=1}^N \Pi_k Q_{t-\frac{1}{2},k}.
\end{equation*}
The complete description of FedQ-SynQ in the matrix-vector form is summarized in Algorithm~\ref{Alg:FedQ_SynQ}.
\begin{algorithm}[!htb]
    \caption{FedQ-SynQ}
    \begin{algorithmic}
    \label{Alg:FedQ_SynQ}
        \STATE \textbf{Initialization:} $\langle \SM,\AM,\RM,\PM,\gamma,\{\SM_k\}_{k=1}^N\rangle$ as the federated control problem, $~Q_{0,k}=Q_0$ as the initial Q functions, $k\in[N]$.
        \FOR{ $r=0$ {\bfseries to} $R-1$ }
                \FOR{ $k=1$ {\bfseries to} $N$ }
                    \FOR{ $e=1$ {\bfseries to} $E$ }
                        \STATE $t\leftarrow rE+e$.
                        \STATE $V_{t-1,k}\leftarrow \max_a Q_{t-1,k}(\cdot,a)$.
                        \STATE $Q_{t,k}=Q_{t-\frac{1}{2},k}\leftarrow\Pi_k\left[(1-\eta) Q_{t-1,k} + \eta \left( \RM + \gamma P_{t-1,k} V_{t-1,k}\right)\right]+\Pi_k^c Q_{t-1,k}$.
                    \ENDFOR
                \ENDFOR
            \STATE $Q_{(r+1)E}\leftarrow\Lambda\sum_{k=1}^N \Pi_k Q_{(r+1)E,k}$.
            \FOR{ $k=1$ {\bfseries to} $N$ }
                \STATE $Q_{(r+1)E,k}\leftarrow Q_{(r+1)E}$.
            \ENDFOR
        \ENDFOR
        \STATE \textbf{return} $Q_{T}=Q_{RE}$.
    \end{algorithmic}
\end{algorithm}
\paragraph{Error Decomposition:}
Now, we are ready to analyze error $\Delta_t$ of FedQ-SynQ.
\begin{equation*}
\begin{aligned}
    \Delta_t&=Q^\star-Q_t\\
    &=Q^\star-\Lambda\sum_{k=1}^N \Pi_k Q_{t-\frac{1}{2},k}\\
    &=\Lambda\sum_{k=1}^N \Pi_k (Q^\star-Q_{t-\frac{1}{2},k})\\
    &=\Lambda\sum_{k=1}^N \Pi_k\Bigg\{\Pi_k\Big[(1-\eta) (Q^\star-Q_{t-1,k}) + \eta \left(Q^\star- \RM - \gamma P_{t,k} V_{t-1,k}\right)\Big]+\Pi_k^c (Q^\star-Q_{t-1,k})  \Bigg\}\\
    &=(1-\eta)\Lambda\sum_{k=1}^N \Pi_k (Q^\star-Q_{t-1,k})+\eta\gamma \Lambda\sum_{k=1}^N \Pi_k(PV^\star-P_{t,k}V_{t-1,k})\\
    &=(1-\eta)\Delta_{t-1}+\eta\gamma \Lambda\sum_{k=1}^N \Pi_k \Big[(P-P_{t,k})V_{t-1,k}+P(V^\star-V_{t-1,k}) \Big],
\end{aligned}
\end{equation*}
we apply it recursively,
\begin{equation*}
      \Delta_t  
  =   \underbrace{ (1-\eta)^t \Delta_0}_{=: E_t^1}  +  \underbrace{ \eta\gamma\Lambda\sum_{i=1}^{t}(1-\eta)^{t-i}\sum_{k=1}^N\Pi_k( P  - P_{i,k}) V_{i-1,k} }_{ =: E_t^2} +  \underbrace{ \eta\gamma \Lambda \sum_{i=1}^{t}    (1-\eta)^{t-i}  \sum_{k=1}^N \Pi_k  P  (V^{\star} - V_{i-1,k}) }_{ =: E_t^3},
\end{equation*}
which is similar to the form of error decomposition in \cite{woo2023blessing}.

$E_t^1$ is easy to deal with: $\|E_t^1\|_\infty=(1-\eta)^t
    \left\|   \Delta_0  \right\|_{\infty}  
  \le   \frac{ (1-\eta)^t}{1-\gamma} $.
\paragraph{Bounding $\E_t^2$:}
$E_t^2$ is a summation of a bounded martingale difference process, which can be bounded leveraging Freedman's inequality \cite{freedman1975tail}. 
Here, we use the following form:
\begin{lem}[Theorem~6 in \cite{li2023q}] \label{thm:Freedman}
Suppose that $Y_{n}=\sum_{k=1}^{n}X_{k}\in\mathbb{R}$,
where $\{X_{k}\}$ is a martingale difference sequence bounded by $B>0$.
Define 
$W_{n}:=\sum_{k=1}^{n}\mathbb{E}\left[X_{k}^{2}|\FM_{k-1}\right]$,
where $\FM_k:=\sigma\left(X_1,\cdots,X_k\right)$. 
Suppose that $W_{n}\leq\sigma^{2}$ holds deterministically for some $\sigma^2>0$.
Then, with probability at least $1-\delta$,
we have
\begin{equation}
\left|Y_{n}\right|\leq\sqrt{8\sigma^2\log\frac{2}{\delta}}+\frac{4}{3}B\log\frac{2}{\delta}.\label{eq:Freedman-random}
\end{equation}
\end{lem}
For $(s,a)\in\SM\times\AM$, we can write $E_t^2(s,a)$ as
\begin{equation*}
      E_t^2(s,a) = \frac{\eta\gamma}{N(s)}  \sum_{i=1}^{t} (1-\eta)^{t-i}  \sum_{k:\ind_k(s)=1} (P(s,a) - P_{i,k}(s,a) ) V_{i-1,k}=:\frac{\eta\gamma}{N(s)}  \sum_{i=1}^{t}\sum_{k\in I(s)} z_{i,k}^{(t)}(s,a),
\end{equation*}
which satisfies the condition of Freedman's inequality, where $z_{i,k}^{(t)}(s,a)=(1-\eta)^{t-i}(P(s,a) - P_{i,k}(s,a) ) V_{i-1,k}$ for $(s,a)\in\SM_k\times\AM$, and $I(s):=\{k\in[N]:\ind_k(s)=1\}$.
The upper bound of the martingale difference process is given by
\begin{equation*}
\begin{aligned}
          B_t(s,a) &:= \max_{k\in I(s),i\in[t]} |z_{i,k}^{(t)}(s,a)|\\
          &\leq \max_{k\in I(s),i\in[t]}\left|(P(s,a) - P_{i,k}(s,a) ) (V_{i-1,k}-\frac{1}{2(1-\gamma)}\mathbf{1}_{|\SM|})\right| \\
          &\leq \max_{k\in I(s),i\in[t]}\big(\| P(s,a)\|_{1} +  \| P_{i,k}(s,a)\|_1 \big) \| V_{i-1,k}-\frac{1}{2(1-\gamma)}\mathbf{1}_{|\SM|}\|_{\infty} \\
          &\leq \frac{1}{1-\gamma},  
\end{aligned}
\end{equation*}

The variance term is given by 
\begin{equation*}
\begin{aligned}
    W_t(s,a)&=\sum_{i=1}^t\sum_{k\in I(s)}\Var\left(z_{i,k}^{(t)}(s,a)|\{V_{j,l}\}_{j<i,l\in I(s)} \right)\\
    &=\sum_{i=1}^t\sum_{k\in I(s)}(1-\eta)^{2(t-i)}\Big[P(s,a) V_{i-1,k}^2-(P(s,a)V_{i-1,k})^2 \Big]\\
    &\leq \sum_{i=1}^t\sum_{k\in I(s)}(1-\eta)^{2(t-i)}\| P(s,a)\|_{1} \| V_{i-1,k}\|_{\infty}^2\\
    &\leq \frac{N(s)}{(1-\gamma)^2}\sum_{i=0}^{\infty} (1-\eta)^{2i}\\
    &=\frac{N(s)}{(1-\gamma)^2} \frac{1}{1-(1-\eta)^2}\\
    &\leq \frac{N(s)}{\eta(1-\gamma)^2}\\
    &=:\sigma^2(s).
\end{aligned}
\end{equation*}
By applying the Freedman's inequality and utilizing the union bound over $s\in\SM,a\in\AM,t\in[T]$, we obtain, with probability at least $1-\frac{\delta}{3}$ , for all $s\in\SM,a\in\AM,t\in[T]$,
\begin{equation*}
\begin{aligned}
    |E_t^2(s,a)|&=\frac{\eta\gamma}{N(s)}  \left|\sum_{i=1}^{t}\sum_{k\in I(s)} z_{i,k}^{(t)}(s,a)\right|\\
    &\leq\frac{\eta\gamma}{N(s)} \left(\sqrt{8\frac{N(s)}{\eta(1-\gamma)^2}\log\frac{6|\SM||\AM|T}{\delta}}+\frac{4}{3}\frac{1}{1-\gamma}\log\frac{6|\SM||\AM|T}{\delta}  \right)\\
    &=\frac{2\gamma}{1-\gamma}\sqrt{\log\frac{6|\SM||\AM|T}{\delta}}\left(\sqrt{\frac{2\eta}{N(s)}}+\frac{2}{3}\eta\sqrt{\log\frac{6|\SM||\AM|T}{\delta}} \right)\\
    &\leq \frac{6\gamma}{1-\gamma}\sqrt{\frac{\eta}{N_{\min}}\log\frac{6|\SM||\AM|T}{\delta}},
\end{aligned}
\end{equation*}
where the last inequality holds since we will choose the step size $\eta\leq\frac{9}{2}N_{\min}\left(\log\frac{6|\SM||\AM|T}{\delta} \right)^{-1}$. Hence,
\begin{equation*}
    \|E_t^2\|_\infty\leq \frac{6\gamma}{1-\gamma}\sqrt{\frac{\eta}{N_{\min}}\log\frac{6|\SM||\AM|T}{\delta}}.
\end{equation*}

\paragraph{Bounding $\E_t^3$:} Let $\beta\in\{0,1,\cdots,R-1\}$, we will provide an upper bound for $\|E_t^3\|_\infty$ for all $\beta E\leq t\leq T$ with high probability, and the choice of $\beta$ will be postponed.
For $(s,a)\in\SM\times\AM$,
\begin{equation*}
    \begin{aligned}
        |E_t^3(s,a)|
  &= \frac{\eta \gamma}{N(s)}\left|\sum_{i=0}^{t-1} \sum_{k\in I(s)} (1-\eta)^{t-i-1} P(s,a) ( V^{\star} - V_{i,k}) \right| \\
  &\leq \underbrace{ \frac{\eta\gamma}{N(s)}\left| \sum_{i=0}^{\iota(t)-\beta E-1} \sum_{k\in I(s)} (1-\eta)^{t-i-1} P(s,a) (V^{\star} - V_{i,k}) \right| }_{=:E_t^{3a}(s,a)}
  + \underbrace{ \frac{\eta\gamma }{N(s)}\left|\sum_{i=\iota(t)-\beta E}^{t-1} \sum_{k\in I(s)} (1-\eta)^{t-i-1} P(s,a) ( V^{\star} - V_{i,k}) \right |}_{=:E_t^{3b}(s,a)}.
    \end{aligned}
\end{equation*}
We will provide upper bounds for $E_t^{3a}$ and $E_t^{3b}$ separately.
\begin{equation*}
    \begin{aligned}
  E_t^{3a}(s,a)
  & \leq \frac{\eta\gamma}{N(s)} \sum_{i=0}^{\iota(t)-\beta E-1} \sum_{k\in I(s)}(1-\eta)^{t-i-1} \| P(s,a)\|_{1} \| V^{\star}-V_{i,k}\|_{\infty}  \\
&  \leq \frac{ \eta \gamma}{1-\gamma} (1-\eta)^{\beta E+(t-\iota(t))} \sum_{i=0}^{\infty} (1-\eta)^{i}\\
  &\le \frac{ \gamma}{1-\gamma} (1-\eta)^{\beta E}.
    \end{aligned}
\end{equation*}

\begin{equation*}
    \begin{aligned}
  E_t^{3b}(s,a)
  & \leq \frac{\eta\gamma}{N(s)} \sum_{i=\iota(t)-\beta E}^{t-1} \sum_{k\in I(s)}(1-\eta)^{t-i-1} \| P(s,a)\|_{1} \| V^{\star}-V_{i,k}\|_{\infty} \\
&  \leq \frac{\eta\gamma}{N(s)} \sum_{i=\iota(t)-\beta E}^{t-1} \sum_{k\in I(s)}(1-\eta)^{t-i-1}  \| V^{\star}-V_{i,k}\|_{\infty},
    \end{aligned}
\end{equation*}
we will first give upper bound and lower bound for $ V^{\star}(s)-V_{i,k}(s)$ separately for all $s\in\SM$. 
Let us define $a^{\star}(s) := \argmax_{a \in \AM} Q^{\star}(s,a), a_{i,k}(s) := \argmax_{a \in\AM} Q_{i,k}(s,a), a_i(s) := \argmax_{a \in \AM}Q_i(s,a)$, and $ d^{k}_{i,j}:=Q_{j,k}-Q_{i,k}$.
\begin{equation*}
    \begin{aligned}
  V^{\star}(s)-V_{i,k}(s)&=Q^\star(s,a^\star(s))-Q_{i,k}(s,a_{i,k}(s))\\
  &\leq Q^\star(s,a^\star(s))-Q_{i,k}(s,a^\star(s))\\
  &=\left(Q^\star(s,a^\star(s))-Q_{\iota(i)}(s,a^\star(s))\right) +\left(Q_{\iota(i)}(s,a^\star(s))-Q_{i,k}(s,a^\star(s))\right)\\
  &=\Delta_{\iota(i)}(s,a^\star(s))-d^k_{\iota(i),i}(s,a^\star(s)),
    \end{aligned}
\end{equation*}
\begin{equation*}
    \begin{aligned}
    V^{\star}(s) - V_{i,k}(s)&= Q^{\star}(s, a^{\star}(s)) - Q_{i,k}(s, a_{i,k}(s)) \\
    &\geq Q^{\star}(s, a_{\iota(i)}(s)) -Q_{i,k}(s, a_{i,k}(s))\\
    &=\left(Q^{\star}(s, a_{\iota(i)}(s)) - Q_{\iota(i)}(s, a_{\iota(i)}(s))\right) +  \left(Q_{\iota(i)}(s, a_{\iota(i)}(s)) - Q_{i,k}(s, a_{i,k}(s))\right)\\
    &\geq \left(Q^{\star}(s, a_{\iota(i)}(s)) - Q_{\iota(i)}(s, a_{\iota(i)}(s))\right) +  \left(Q_{\iota(i)}(s, a_{i,k}(s)) - Q_{i,k}(s, a_{i,k}(s))\right)\\
    &=\Delta_{\iota(i)}(s,a_{\iota(i)}(s))-d^k_{\iota(i),i}(s,a_{i,k}(s)),
    \end{aligned}
\end{equation*}
combine the results, we have
\begin{equation*}
    \|V^{\star} - V_{i,k}\|_\infty\leq\|\Delta_{\iota(i)}\|_\infty+\|d^k_{\iota(i),i}\|_\infty=\|\Delta_{\iota(i)}\|_\infty+\|Q_{i,k}-Q_{\iota(i)}\|_{\SM_k},
\end{equation*}
we need to bound $\|Q_{i,k}-Q_{\iota(i)}\|_{\SM_k}$, which represents the local update magnitude for agent $k$ within a communication round, and it can be easily handled. 
We start in a similar way as we expanded $\Delta_t$ at the beginning, for all $j\in\{\iota(i), \iota(i)+1,\cdots i-1\}$, $(s,a)\in\SM_k\times\AM$
\begin{equation*}
    \begin{aligned}
        Q_{j+1,k}(s,a)-Q_{j,k}(s,a)&=(1-\eta)Q_{j,k}(s,a)+\eta\left(\RM(s,a)+\gamma P_{j+1,k}(s,a) V_{j,k}\right)   -Q_{j,k}(s,a) \\
        &=\eta\left(\RM(s,a)+\gamma P_{j+1,k}(s,a) V_{j,k}\right)-\eta Q_{j,k}(s,a) \\
        &=\eta\left(\RM(s,a)+\gamma P_{j+1,k}(s,a) V_{j,k}-Q^\star(s,a)\right)+\eta \left(Q^\star(s,a)- Q_{j,k}(s,a) \right)\\
        &=\eta\left(\gamma P_{j+1,k}(s,a)(V_{j,k}-V^\star)+\gamma(P_{j+1,k}(s,a)-P(s,a))V^\star+\Delta_{j,k}(s,a)\right),
    \end{aligned}
\end{equation*}
where $\Delta_{j,k}:=Q^\star- Q_{j,k}$ is the local error term, we apply the result recursively and obtain
\begin{equation*}
    \begin{aligned}
        |Q_{i,k}(s,a)-Q_{\iota(i)}(s,a)|&= \eta\left| \sum_{j=\iota(i)}^{i-1}\Big(\gamma P_{j+1,k}(s,a)(V_{j,k}-V^\star)+\gamma(P_{j+1,k}(s,a)-P(s,a))V^\star+\Delta_{j,k}(s,a)\Big)
  \right|  \\
        &\leq\eta\gamma\sum_{j=\iota(i)}^{i-1}  \|P_{j+1,k}(s,a)\|_1 \|V_{j,k}-V^\star\|_\infty +\eta\gamma \left|\sum_{j=\iota(i)}^{i-1}(P_{j+1,k}(s,a)-P(s,a))V^\star \right|+\eta\sum_{j=\iota(i)}^{i-1}|\Delta_{j,k}(s,a)|\\
        &\leq 2\eta \sum_{j=\iota(i)}^{i-1}\|\Delta_{j,k}\|_\infty+\eta\gamma\left|\sum_{j=\iota(i)}^{i-1}(P_{j+1,k}(s,a)-P(s,a))V^\star \right|,
    \end{aligned}
\end{equation*}
where the last inequality holds due to Lemma~\ref{lem:max}.
The second term can be bounded using the Hoeffding's inequality for the sum of independent bounded random variables (Theorem 2.2.6 in \cite{vershyninHighdimensionalProbabilityIntroduction2018}): $|(P_{j+1,k}(s,a)-P(s,a))V^\star|\leq \frac{1}{1-\gamma}$, combine Hoeffding's inequality with union bound over $s\in\SM_k, a\in\AM, t\in[T]$ and $k\in[N]$, we have, with probability at least $1-\frac{\delta}{3}$,
\begin{equation*}
    \begin{aligned}
        \eta\gamma\left|\sum_{j=\iota(i)}^{i-1}(P_{j+1,k}(s,a)-P(s,a))V^\star \right|&\leq\eta\gamma\sqrt{\frac{1}{2}\sum_{j=\iota(i)}^{i-1}\frac{1}{(1-\gamma)^2}\log\frac{6|\SM||\AM|TN}{\delta}}\\
        &\leq \frac{\eta\gamma}{1-\gamma}\sqrt{\frac{E-1}{2}\log\frac{6|\SM||\AM|TN}{\delta}}.
    \end{aligned}
\end{equation*}
As for the first term $2\eta \sum_{j=\iota(i)}^{i-1}\|\Delta_{j,k}\|_\infty$, we aim to bound $\|\Delta_{j,k}\|_\infty$ using $\|\Delta_{\iota(i)}\|_\infty$, which will allow us to eventually derive a recursive relationship for the error term $\Delta_t$.

\paragraph{Bounding $\|\Delta_{j,k}\|_\infty$ using $\|\Delta_{\iota(i)}\|_\infty$:}
Once again, we perform the same expansion as at the beginning to establish the relationship between $\Delta_{j,k}$ and $\Delta_{\iota(i)}$ for $j\in\{\iota(i), \iota(i)+1,\cdots i-1\}$. 
For all $(s,a)\in\SM_k\times\AM$,
\begin{equation*}
    \begin{aligned}
        &\Delta_{j,k}(s,a)\\
        =&Q^\star(s,a)-Q_{j,k}(s,a)\\
        =&(1-\eta)\Delta_{j-1,k}(s,a)+\eta\gamma\Big[(P(s,a)-P_{j,k}(s,a))V^\star+P_{j,k}(s,a)(V^\star-V_{j-1,k}) \Big]\\
        =&\underbrace{(1-\eta)^{j-\iota(i)}\Delta_{\iota(i)}(s,a)}_{=:F_j^1(s,a)}+\underbrace{\eta\gamma\sum_{l=\iota(i)}^{j-1} (1-\eta)^{j-l-1}(P(s,a)-P_{l+1,k}(s,a))V^\star}_{=:F_j^2(s,a)}+\underbrace{\eta\gamma\sum_{l=\iota(i)}^{j-1}(1-\eta)^{j-l-1}P_{l+1,k}(s,a)(V^\star-V_{l,k})}_{=:F_j^3(s,a)},
    \end{aligned}
\end{equation*}
$F_j^1(s,a)$ is easy to deal with: $|F_t^1(s,a)|\leq(1-\eta)^{j-\iota(i)}
    \left\|   \Delta_{\iota(i)}  \right\|_{\infty}$.
    
$F_j^2(s,a)$ can be bounded using the Hoeffding's inequality for the sum of independent bounded random variables (Theorem 2.2.6 in \cite{vershyninHighdimensionalProbabilityIntroduction2018}): $|(P(s,a)-P_{l+1,k}(s,a))V^\star|\leq \frac{1}{1-\gamma}$, combine Hoeffding's inequality with union bound over $s\in\SM_k, a\in\AM, t\in[T]$ and $k\in[N]$, we have, with probability at least $1-\frac{\delta}{3}$,
\begin{equation*}
    \begin{aligned}
        |F_j^2(s,a)|&=\eta\gamma\left|\sum_{l=\iota(i)}^{j-1} (1-\eta)^{j-l-1}(P(s,a)-P_{l+1,k}(s,a))V^\star\right|\\
        &\leq\eta\gamma\sqrt{\frac{1}{2}\sum_{l=\iota(i)}^{j-1}(1-\eta)^{2(j-l-1)}\frac{1}{(1-\gamma)^2}\log\frac{6|\SM||\AM|TN}{\delta}}\\
        &\leq \frac{\eta\gamma}{1-\gamma}\sqrt{\frac{1}{2}\sum_{l=0}^{\infty}(1-\eta)^{2l}\log\frac{6|\SM||\AM|TN}{\delta}}\\
        &\leq \frac{\gamma}{1-\gamma}\sqrt{\frac{\eta}{2}\log\frac{6|\SM||\AM|TN}{\delta}}\\
        &=:\rho.
    \end{aligned}
\end{equation*}
As for $F_j^3(s,a)$,
\begin{equation*}
    \begin{aligned}
        |F_j^3(s,a)|&=\eta\gamma\left|\sum_{l=\iota(i)}^{j-1}(1-\eta)^{j-l-1}P_{l+1,k}(s,a)(V^\star-V_{l,k})\right|\\
        &\leq \eta\gamma\sum_{l=\iota(i)}^{j-1}(1-\eta)^{j-l-1} \|P_{l+1,k}(s,a)\|_{1}\|V^\star-V_{l,k} \|_\infty\\
        &\leq\eta\gamma\sum_{l=\iota(i)}^{j-1}(1-\eta)^{j-l-1}\|\Delta_{l,k} \|_\infty.
    \end{aligned}
\end{equation*}
Combine the previous upper bounds, we obtain
\begin{equation*}
    \begin{aligned}
        \|\Delta_{j,k}\|_{\SM_k}\leq(1-\eta)^{j-\iota(i)}
    \left\|   \Delta_{\iota(i)}  \right\|_{\infty}+\rho+\eta\gamma\sum_{l=\iota(i)}^{j-1}(1-\eta)^{j-l-1}\|\Delta_{l,k} \|_\infty.
    \end{aligned}
\end{equation*}
Given $\eta E\leq\frac{1}{2}$, we will prove by induction for $m\in\{0,1,\cdots,E-1\}$ that,
\begin{equation*}
    \begin{aligned}
        \|\Delta_{\iota(i)+m,k}\|_{\infty}\leq \left\|   \Delta_{\iota(i)}  \right\|_{\infty}+2\rho,
    \end{aligned}
\end{equation*}
note that 
\begin{equation*}
    \begin{aligned}
        \|\Delta_{\iota(i)+m,k}\|_{\infty}&=\max\{\|\Delta_{\iota(i)+m,k}\|_{\SM_k}, \|\Delta_{\iota(i)+m,k}\|_{\SM_k^c}\}\\
        &=\max\{\|\Delta_{\iota(i)+m,k}\|_{\SM_k}, \|\Delta_{\iota(i)}\|_{\SM_k^c}\}\\
        &\leq\max\{\|\Delta_{\iota(i)+m,k}\|_{\SM_k}, \|\Delta_{\iota(i)}\|_{\infty}\},
    \end{aligned}
\end{equation*}
hence $\|\Delta_{\iota(i)+m,k}\|_{\SM_k}\leq \left\|   \Delta_{\iota(i)}  \right\|_{\infty}+2\rho$ is equivalent to $\|\Delta_{\iota(i)+m,k}\|_{\infty}\leq \left\|   \Delta_{\iota(i)}  \right\|_{\infty}+2\rho$.

When $m=0$, $\|\Delta_{\iota(i),k}\|_{\infty}=\|\Delta_{\iota(i)}\|_{\infty}<\|\Delta_{\iota(i)}\|_{\infty}+2\rho$ holds.
When $m\in[E-1]$, suppose $\|\Delta_{\iota(i)+l,k}\|_{\SM_k}\leq \left\|   \Delta_{\iota(i)}  \right\|_{\infty}+2\rho$ holds for $l\in\{0,1,\cdots,m-1\}$, 
\begin{equation*}
    \begin{aligned}
        \|\Delta_{\iota(i)+m,k}\|_{\SM_k}&\leq(1-\eta)^{m}
    \left\|   \Delta_{\iota(i)}  \right\|_{\infty}+\rho+\eta\gamma\sum_{l=\iota(i)}^{\iota(i)+m-1}(1-\eta)^{\iota(i)+m-l-1}\|\Delta_{l,k} \|_\infty \\
    &\leq  (1-\eta)^{m}
    \left\|   \Delta_{\iota(i)}  \right\|_{\infty}+\rho+\eta\gamma \left(\left\|   \Delta_{\iota(i)}  \right\|_{\infty}+2\rho\right)\sum_{l=0}^{m-1}(1-\eta)^{l}\\
    &= \Big\{(1-\eta)^{m}+\gamma\left[1-(1-\eta)^m\right]  \Big\}\left\| \Delta_{\iota(i)}\right\|_{\infty} +   \Big\{1+2\gamma\left[1-(1-\eta)^m\right]  \Big\} \rho,
    \end{aligned}
\end{equation*}
where $(1-\eta)^{m}+\gamma\left[1-(1-\eta)^m\right]< (1-\eta)^{m}+1-(1-\eta)^m=1$, and for the second term
\begin{equation*}
    \begin{aligned}
        (1-\eta)^m\geq (1-\eta)^E=\left[(1-\eta)^{\frac{1}{\eta}}\right]^{\eta E}\geq \left(\frac{1}{4}\right)^{\frac{1}{2}}=\frac{1}{2},
    \end{aligned}
\end{equation*}
since $\eta E\leq \frac{1}{2}$ (implying $\eta\leq\frac{1}{2}$), hence $\|\Delta_{\iota(i)+m,k}\|_{\SM_k}\leq \left\| \Delta_{\iota(i)}\right\|_{\infty}+2\rho$, which is desired.

Substituting this upper bound back into the previous result, we have:
\begin{equation*}
    \begin{aligned}
        \|Q_{i,k}-Q_{\iota(i)}\|_{\SM_k}&\leq 2\eta \sum_{j=\iota(i)}^{i-1}\|\Delta_{j,k}\|_\infty+\eta\gamma\left|\sum_{j=\iota(i)}^{i-1}(P_{j+1,k}(s,a)-P(s,a))V^\star \right|\\
        &\leq 2\eta (E-1)\left\| \Delta_{\iota(i)}\right\|_{\infty} +\frac{4\eta \gamma(E-1)}{1-\gamma}\sqrt{\frac{\eta}{2}\log\frac{6|\SM||\AM|TN}{\delta}}   +\frac{\eta\gamma}{1-\gamma}\sqrt{\frac{E-1}{2}\log\frac{6|\SM||\AM|TN}{\delta}}\\
        &=2\eta (E-1)\left\| \Delta_{\iota(i)}\right\|_{\infty} + \frac{\eta\gamma}{1-\gamma}\sqrt{\frac{E-1}{2}\log\frac{6|\SM||\AM|TN}{\delta}}\left(  4\sqrt{\eta (E-1)}+1\right)\\
        &\leq 2\eta (E-1)\left\| \Delta_{\iota(i)}\right\|_{\infty} + \frac{3\eta\gamma}{1-\gamma}\sqrt{(E-1)\log\frac{6|\SM||\AM|TN}{\delta}}\\
        &=:2\eta (E-1)\left\| \Delta_{\iota(i)}\right\|_{\infty} +\kappa,
    \end{aligned}
\end{equation*}
therefore,
\begin{equation*}
    \begin{aligned}
  E_t^{3b}(s,a)
&  \leq \frac{\eta\gamma}{N(s)} \sum_{i=\iota(t)-\beta E}^{t-1} \sum_{k\in I(s)}(1-\eta)^{t-i-1}  \left(\|\Delta_{\iota(i)}\|_\infty+\|Q_{i,k}-Q_{\iota(i)}\|_{\SM_k}\right)\\
&\leq \frac{\eta\gamma}{N(s)} \sum_{i=\iota(t)-\beta E}^{t-1} \sum_{k\in I(s)}(1-\eta)^{t-i-1}\Big\{\left[1+2\eta(E-1)\right]\|\Delta_{\iota(i)}\|_\infty+ \kappa \Big\}\\
&\leq \eta\gamma \Big\{\left[1+2\eta(E-1)\right]\max_{\iota(t)-\beta E\leq i\leq t-1}\|\Delta_{\iota(i)}\|_\infty+ \kappa \Big\}\sum_{i=0}^\infty (1-\eta)^i\\
&=\gamma\left[1+2\eta(E-1)\right]\max_{\iota(t)-\beta E\leq i\leq t-1}\|\Delta_{i}\|_\infty+ \gamma\kappa,
    \end{aligned}
\end{equation*}
and finally,
\begin{equation*}
    \begin{aligned}
  \|E_t^{3}\|_\infty\leq \frac{\gamma}{1-\gamma}(1-\eta)^{\beta E}+\gamma\left[1+2\eta(E-1)\right]\max_{\iota(t)-\beta E\leq i\leq t-1}\|\Delta_{i}\|_\infty+ \frac{3\eta\gamma^2}{1-\gamma}\sqrt{(E-1)\log\frac{6|\SM||\AM|TN}{\delta}}.
    \end{aligned}
\end{equation*}

\paragraph{Solving the Recursive  Inequality:}
Substituting the upper bounds for $E_t^1$, $E_t^2$, and $E_t^3$ back into the expression for the error decomposition, we obtain: with probability at least $1-\delta$, for all $\beta E\leq t\leq T$,
\begin{equation*}
    \begin{aligned}
  \|\Delta_t\|_\infty&\leq \|E_t^1\|_\infty+\|E_t^2\|_\infty+\|E_t^3\|_\infty\\
  &\leq \frac{(1-\eta)^t}{1-\gamma}+\frac{6\gamma}{1-\gamma}\sqrt{\frac{\eta}{N_{\min}}\log\frac{6|\SM||\AM|T}{\delta}}
  +\frac{\gamma}{1-\gamma}(1-\eta)^{\beta E}+\\
  &\gamma\left[1+2\eta(E-1)\right]\max_{\iota(t)-\beta E\leq i\leq t-1}\|\Delta_{i}\|_\infty+ \frac{3\eta\gamma^2}{1-\gamma}\sqrt{(E-1)\log\frac{6|\SM||\AM|TN}{\delta}}\\
  &\leq \frac{2(1-\eta)^{\beta E}}{1-\gamma}+\frac{3\gamma}{1-\gamma}\sqrt{\eta\log\frac{6|\SM||\AM|TN}{\delta}}\left(\frac{2}{\sqrt{N_{\min}}}+\sqrt{\eta(E-1)} \right) +\frac{1+\gamma}{2}\max_{\iota(t)-\beta E\leq i\leq t-1}\|\Delta_{i}\|_\infty\\
  &\leq\left(\frac{2(1-\eta)^{\beta E}}{1-\gamma}+ \frac{9\gamma}{1-\gamma}\sqrt{\frac{\eta}{N_{\min}}\log\frac{6|\SM||\AM|TN}{\delta}} \right)+\frac{1+\gamma}{2}\max_{\iota(t)-\beta E\leq i\leq t-1}\|\Delta_{i}\|_\infty\\
  &=: M+\frac{1+\gamma}{2}\max_{\iota(t)-\beta E\leq i\leq t-1}\|\Delta_{i}\|_\infty,
    \end{aligned}
\end{equation*}
where the third and last inequality hold since we assume $\eta(E-1)\leq\min\left\{\frac{1-\gamma}{4\gamma},\frac{1}{N_{\min}}\right\}$.

Let us continue to analyse this recursive inequality. 
We aim to narrow down the range of $t$ in which an inequality holds and obtain a tighter upper bound.
Note that for all $2\beta E\leq t\leq T$, $\iota(t)-\beta E\geq \beta E$, hence
\begin{equation*}
    \begin{aligned}
  \|\Delta_t\|_\infty &\leq M+\frac{1+\gamma}{2}\max_{\iota(t)-\beta E\leq i\leq t-1}\|\Delta_{i}\|_\infty\\
  &\leq M+\frac{1+\gamma}{2}\max_{\iota(t)-\beta E\leq i\leq t-1}\left(M+\frac{1+\gamma}{2}\max_{\iota(i)-\beta E\leq j\leq i-1}\|\Delta_{j}\|_\infty \right)\\
  &\leq \left(1+\frac{1+\gamma}{2}\right)M+\left(\frac{1+\gamma}{2}\right)^2 \max_{\iota(t)-2\beta E\leq i\leq t-1}\|\Delta_{i}\|_\infty.
    \end{aligned}
\end{equation*}
We can apply the inequality recursively $L$ times ($L$ will be determined later, and $L\beta E$ should be less than $T$) and we have for all $L\beta E\leq t\leq T$,
\begin{equation*}
    \begin{aligned}
  \|\Delta_t\|_\infty &\leq \sum_{l=0}^{L-1}\left(\frac{1+\gamma}{2}\right)^l M+\left(\frac{1+\gamma}{2}\right)^L \max_{\iota(t)-L\beta E\leq i\leq t-1}\|\Delta_{i}\|_\infty\\
  &\leq \frac{2M}{1-\gamma}+\left(\frac{1+\gamma}{2}\right)^L\frac{1}{1-\gamma}.
    \end{aligned}
\end{equation*}
We can take $\beta = \left \lfloor \frac{1}{E} \sqrt{\frac{(1-\gamma)T}{2\eta }} \right \rfloor$ and $L = \left \lceil \sqrt{\frac{\eta T}{1-\gamma}} \right \rceil$, at this time, $L\beta E\leq T$, and 
\begin{equation*}
    \begin{aligned}
(1-\eta)^{\beta E}\leq \exp( - \eta \beta E)\leq \exp\left(-\frac{\sqrt{(1-\gamma)\eta T}}{2}\right),
    \end{aligned}
\end{equation*}

\begin{equation*}
    \begin{aligned}
\left(\frac{1+\gamma}{2}\right)^L=\left(1-\frac{1-\gamma}{2}\right)^L \leq \exp\left(-\frac{(1-\gamma)}{2}L \right) \leq \exp\left(-\frac{\sqrt{(1-\gamma)\eta T}}{2}\right),
    \end{aligned}
\end{equation*}

we have the upper bound for $\|\Delta_T\|_\infty$:
\begin{equation*}
    \begin{aligned}
  \|\Delta_T\|_\infty &\leq \frac{2M}{1-\gamma}+\left(\frac{1+\gamma}{2}\right)^L\frac{1}{1-\gamma}\\
  &\leq \frac{1}{(1-\gamma)^2}\left[ 5\exp\left(-\frac{\sqrt{(1-\gamma)\eta T}}{2}\right)+18\gamma\sqrt{\frac{\eta}{N_{\min}}\log\frac{6|\SM||\AM|TN}{\delta}}\right].
    \end{aligned}
\end{equation*}
For any $\varepsilon\in(0,1)$, if we choose $\eta$ and $T$ such that $\exp\left(-\frac{\sqrt{(1-\gamma)\eta T}}{2}\right) \leq \frac{\varepsilon(1-\gamma)^2}{10}$ and $\gamma\sqrt{\frac{\eta}{N_{\min}}\log\frac{6|\SM||\AM|TN}{\delta}} \leq \frac{\varepsilon(1-\gamma)^2}{36}$, we can ensure that $\|\Delta_T\|_\infty \leq \varepsilon$, at this time, $\eta\leq \frac{1}{1296}N_{\min} \varepsilon^2 (1-\gamma)^4  \left(\log{\frac{6|\SM||\AM|TN}{\delta}}\right)^{-1}$ and $T\geq\frac{1296}{N_{\min}\varepsilon^2(1-\gamma)^5 } \left(\log\frac{10}{\varepsilon(1-\gamma)^2}\right)^2 \log{\frac{6|\SM||\AM|TN}{\delta}}$.
If we ignore the logarithmic terms, we choose $T=\widetilde{O}\left(\frac{1}{N_{\min}\varepsilon^2(1-\gamma)^5 }\right)$ and $\eta=\widetilde{\Omega}\left(N_{\min} \varepsilon^2 (1-\gamma)^4\right)$, which is desired.
\end{proof}

\newpage
\section{Details of Experiments}
\subsection{Construction of environments}
\label{App:env_construction}
\textbf{RandomMDP.}
For a RandomMDP parameterized by $(N,K_S,E_S,N_S)$, the state space of corresponding global MDP $\MM=\langle \SM,\AM,\PM,\RM,\gamma \rangle$ has a size of 
\begin{equation*}
    |\SM| = K_S*N + E_S*N * N / N_S.
\end{equation*}
For $E_S*N*N/N_S$ states, we assign each state to $N_S$ agents.
In this way, the $k$-th agent is assigned a restricted region $\SM_k$ with a size of
\begin{equation*}
    |\SM_k|=K_S+E_S
\end{equation*}
in expectation.
If not specified $p_{\max}$, the transition dynamics $\PM$ of $\MM$ is randomly generated.
When we specify $p_{\max}$ to control degrees of connection among $\{\SM_k\}_{k=1}^N$, $E_S$ is by default set to $0$ for simplification in generating $\PM$.
Specifically, $\PM$ is the combination of two randomly generated transition dynamics ($\PM_{in}^k$ and $\PM_{out}^k$ for the $k$-th agent), where $\PM_{in}^k$ represent transitions from $\SM_k$ to $\SM_k$ and $\PM_{out}^k$ represent transitions f rom $\SM_k$ to $\SM\backslash\SM_k$.
Therefore, $\PM$ is generated as follows:
\begin{equation*}
    \PM^k = (1-p_{\max})\PM_{in}^k\oplus p_{\max}\PM_{out}^k,
\end{equation*}
where $\PM^k$ represents the submatrix of $\PM$ restricted within $\SM_k$, and $\oplus$ represents the concatenation of transitions to different regions.
In Figure \ref{fig:FedQ-SyncQ}, the first row sets $K_S=20,E_S=20,N_S/N=0.6$, and the second row sets $K_S=20,N=10$.

\begin{figure}[!htb]
    \centering
    \includegraphics[width=0.5\textwidth]{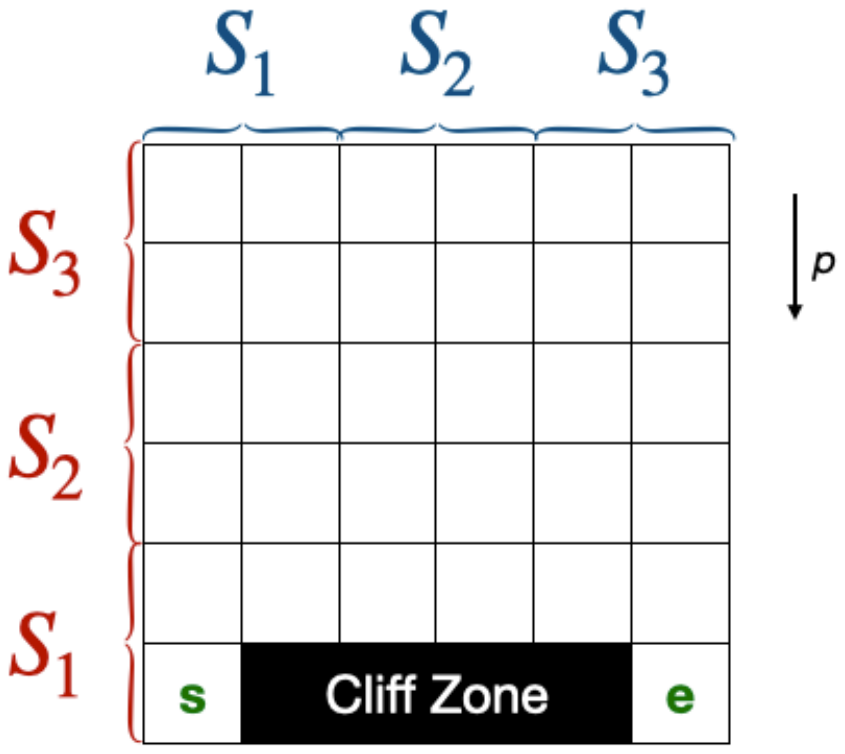}
    
    \caption{Examples of WindyCliff with a $6\times 6$ state space: Red $\{S_k\}_{k=1}^3$ represents horizontal splitting; blue $\{S_k\}_{k=1}^3$ represents vertical splitting; $p$ indicates the power of wind glowing downwards.}
    \label{WindyCliff_example}
\end{figure}

\textbf{WindyCliff.}
As shown in Figure \ref{WindyCliff_example}, the state space of WindyCliff is composed of "Cliff Zone" and "Land Zone".
In terms of reward function $\RM$, the agent gets values of $1.0$, $-0.01$ and $-0.1$ when it reaches $s=e$, "Land Zone" other than $e$, and "Cliff Zone" respectively.
Agents have four actions corresponding to going at directions of \{"up", "down", "right", "left"\}.
There is a probability of $p$ for the glowing wind with a power of $p\in(0,1)$ to override agents' actions as "down".
Ways of splitting the state space have been shown in Figure \ref{WindyCliff_example}.
Specifically, we consider two splitting directions: horizontal splitting (h), and vertical splitting (v).
In Figure \ref{fig:FedQ-SyncQ}, the third row considers $N=5$ agents located in $10\times 10$ state space which is horizontally split.

\subsection{Selection of hyper-parameters}
\label{App:exper_hyper}
In terms of the application of FedQ-SynQ, we have to specify choices of learning rates $\{\eta_t\}_{t=1}^\infty$ and batch-size $b$ for local generators.

\textbf{RandomMDP.} 
The learning rate $\eta$ is uniformly set to $0.5$, while batch size $b$ is set to $5$.

\textbf{WindyCliff.}
The learning rate $\eta$ is uniformly set to $0.5$, while batch size $b$ is set to $10$.

\begin{figure*}
\centering
    \includegraphics[width=0.6\textwidth]{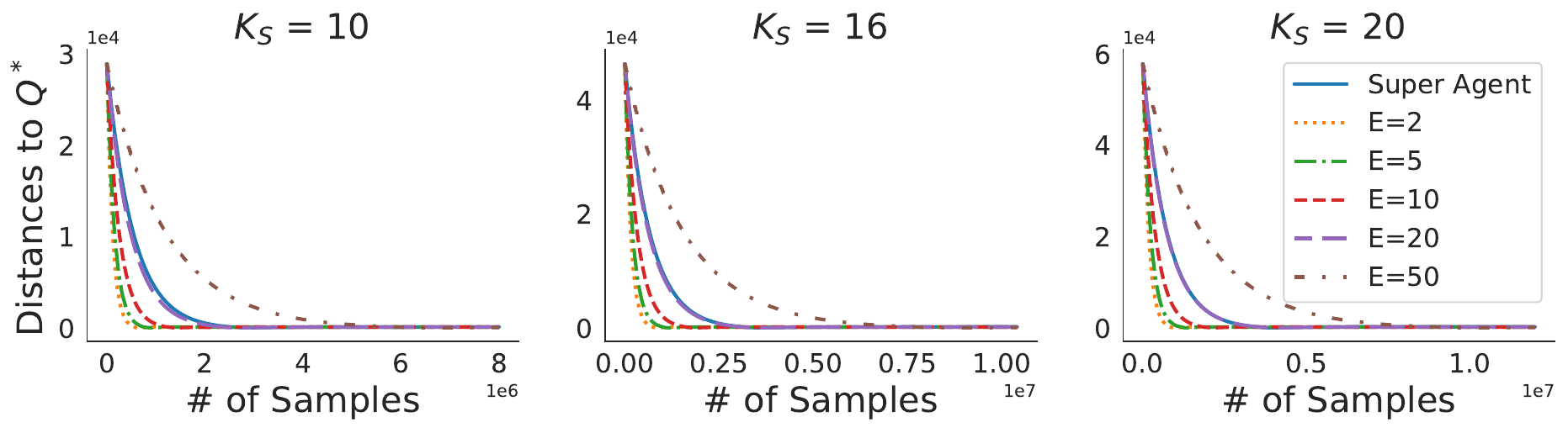}
    \includegraphics[width=0.6\textwidth]{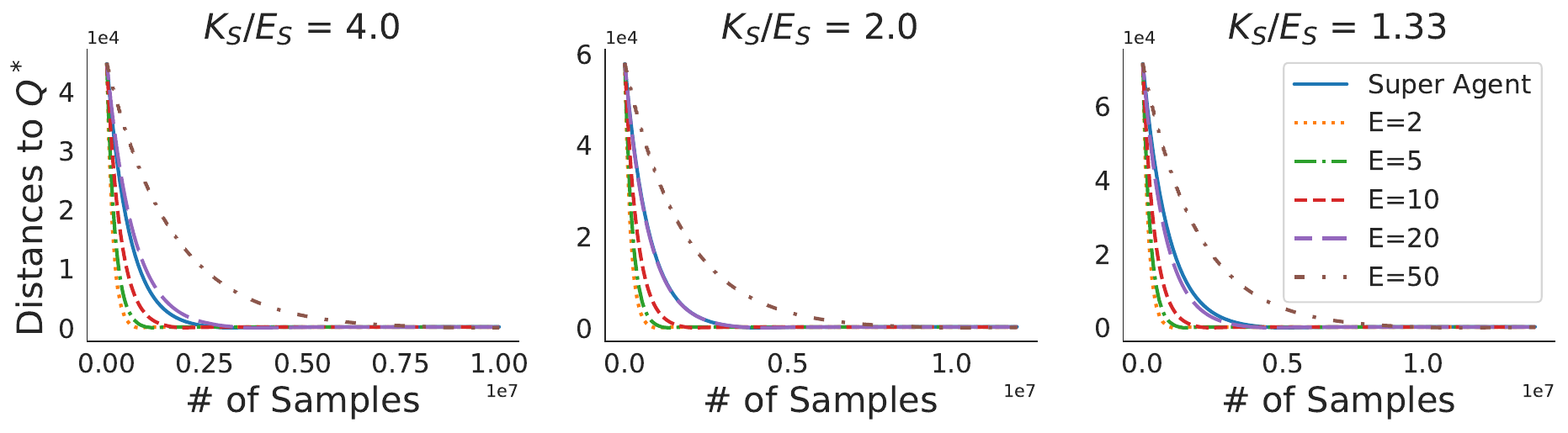}
    \includegraphics[width=0.6\textwidth]{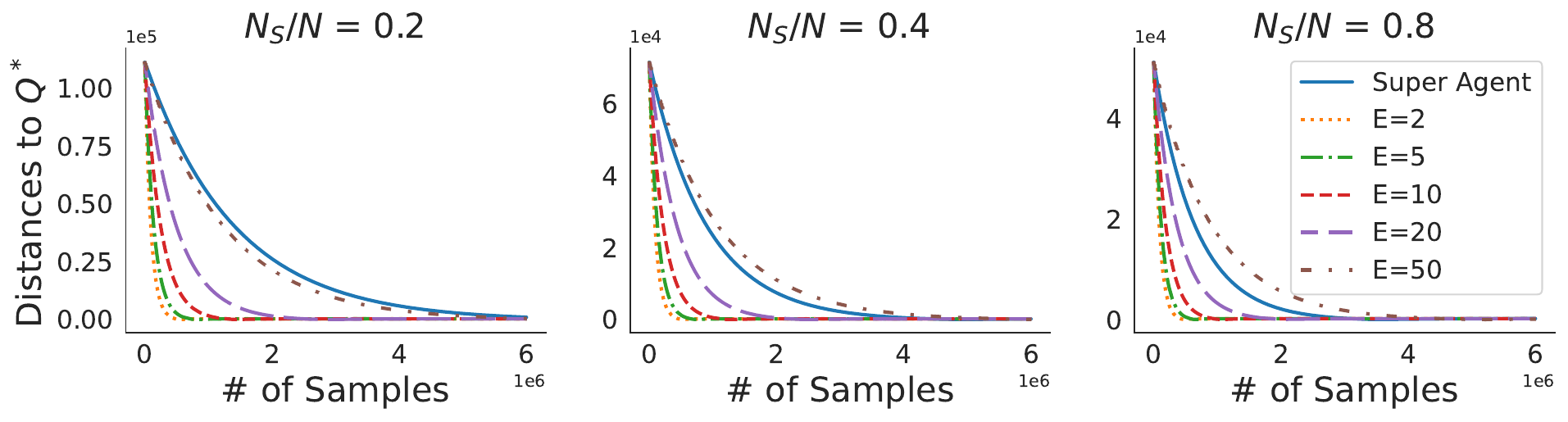}
     \includegraphics[width=0.8\textwidth]{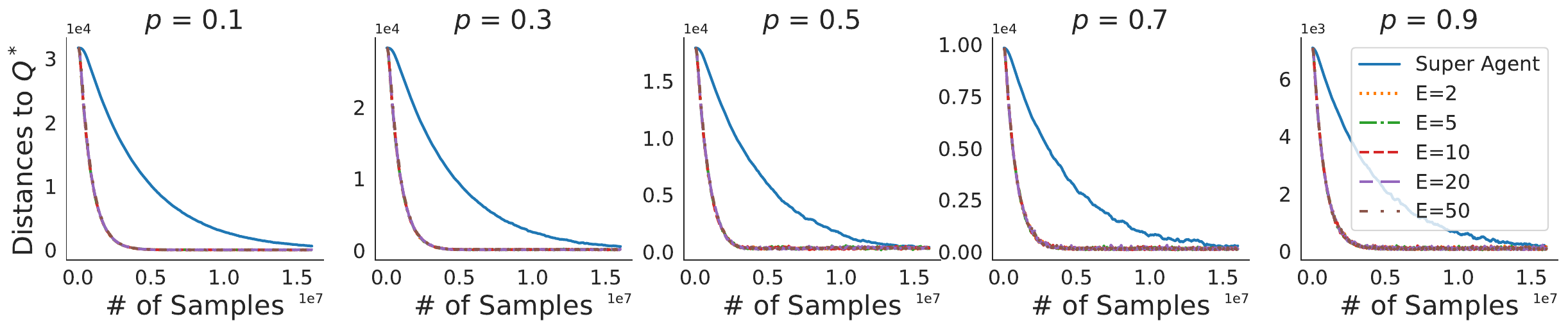}
     \includegraphics[width=0.8\textwidth]{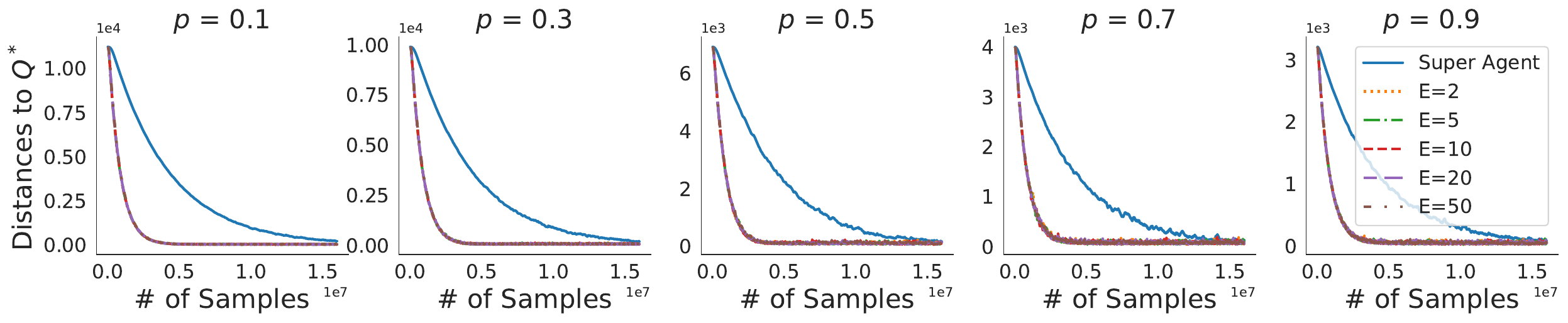}
      \includegraphics[width=0.8\textwidth]{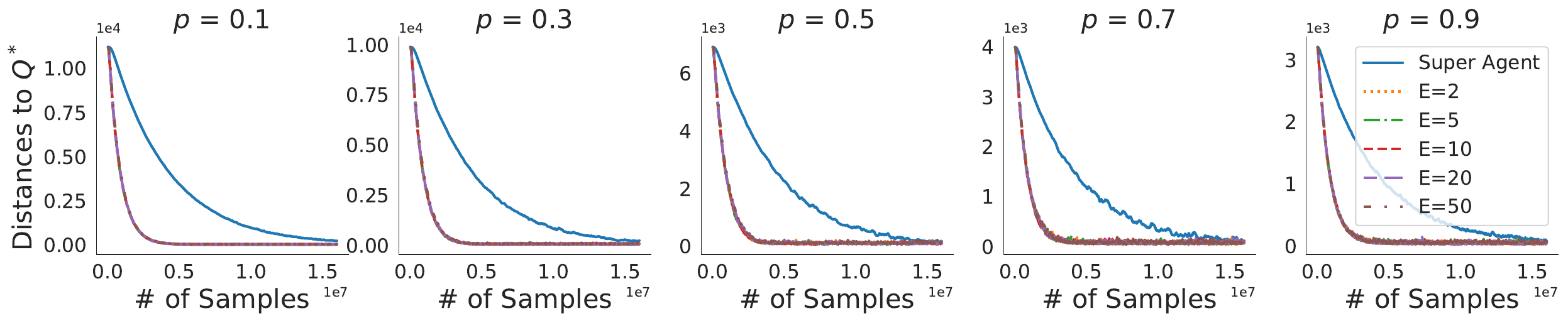}
      \caption{
      Convergence of FedQ-SynQ in different environments:
      the first row with $E_S=0.5,K_S,N_S=3,N=5$;
      the second row with $K_S=20,N_S=3,N=5$;
      the third row with $K_S=10,E_S=5,N=10$;
      the forth row with $10\times 10$ state space, $N=5$ agents and splitting direction as $v$;
      the fifth row with $6\times 6$ state space, $N=3$ agents and splitting direction as $v$;
      the sixth row with $6\times 6$ state space, $N=3$ agents and splitting direction as $h$.
      }
      \label{fig:exper_more}
\end{figure*}
\subsection{More Numerical Experiments}
\label{App:exper_more}
We construct additional numerical experiments to evaluate how structure of federated control problems affect the efficiency of FedQ-SyncQ.
Specifically, we consider following settings in environments of RandomMDP and WindyCliff:
\begin{itemize}
    \item RandomMDP with different sizes of $K_S$ with fixed values of $(N_S,N)$ and fixed ratios of $K_S/E_S$;
    \item RandomMDP with different ratios of $K_S/E_S$ with fixed value of $(K_S,N_S,N)$ and fixed value of $K_S$;
    \item RandomMDP with different ratios of $N/N_S$ with fixed values of $(K_S,E_S,N)$;
    \item WindyCliff with different sizes of state space $6\times 6$, different number of agents $N$, and different splitting directions.
\end{itemize}
As shown in Figure \ref{fig:exper_more}, FedQ-SynQ exhibits smaller sample complexity in converging to globally optimal Q functions.
The first row with fixed ratio of $K_S/E_S$ but different values of $K_S$ has shown that the speedup effect of FedQ-SynQ remains at different problem sizes;
the second row with fixed value of $K_S$ but different values of $E_S$ has shown that FedQ-SynQ achieves slightly greater speed-up when an agent has more states sharing with other agents;
the third row with fixed values of $(K_S,E_S,N)$ but different ratios of $N_S/N$ indicates that the speedup effect of FedQ-SynQ decreases when the difference between sizes of restricted state subset and global set diminishes.
Remaining three rows of experiments on WindyCliffs have demonstrated an obvious speedup effect of FedQ-SynQ on different sizes of state space and different values of wind power.

\end{document}